\documentclass[lettersize,journal]{IEEEtran}
\usepackage{amsmath,amsfonts}
\usepackage{algorithmic}
\usepackage{array}
\usepackage[caption=false,font=normalsize,labelfont=sf,textfont=sf]{subfig}
\usepackage{textcomp}
\usepackage{stfloats}
\usepackage{url}
\usepackage{verbatim}
\usepackage{graphicx}

\usepackage{microtype}
\usepackage{graphicx}
\usepackage{booktabs} %
\usepackage{mathtools} %
\usepackage{amssymb}
\usepackage{amsthm}
\usepackage{color}
\usepackage[inline]{enumitem}
\usepackage{amsmath,amsfonts}
\usepackage{algorithmic}
\usepackage{algorithm}
\usepackage{array}
\usepackage[caption=false,font=normalsize,labelfont=sf,textfont=sf]{subfig}
\usepackage{textcomp}
\usepackage{stfloats}
\usepackage{url}
\usepackage{verbatim}
\usepackage{graphicx}
\usepackage{cite}
\usepackage{adjustbox}

\newtheorem{thm}{Theorem}

\newcommand{\diag}{\mathrm{diag}}

\newcommand{\N}{\mathbb{N}}

\usepackage{amsmath,amsfonts,bm}

\def\eqref#1{equation~\ref{#1}}

\def\1{\bm{1}}

\def\vtheta{{\bm{\theta}}}

\def\vc{{\bm{c}}}

\def\vh{{\bm{h}}}

\def\vo{{\bm{o}}}

\def\vx{{\bm{x}}}
\def\vy{{\bm{y}}}

\def\mD{{\bm{D}}}

\def\mF{{\bm{F}}}

\def\mI{{\bm{I}}}
\def\mJ{{\bm{J}}}

\def\mM{{\bm{M}}}

\def\mP{{\bm{P}}}
\def\mQ{{\bm{Q}}}

\def\mS{{\bm{S}}}

\def\mU{{\bm{U}}}

\def\mW{{\bm{W}}}
\def\mX{{\bm{X}}}

\def\mPhi{{\bm{\Phi}}}
\def\mLambda{{\bm{\Lambda}}}

\DeclareMathAlphabet{\mathsfit}{\encodingdefault}{\sfdefault}{m}{sl}
\SetMathAlphabet{\mathsfit}{bold}{\encodingdefault}{\sfdefault}{bx}{n}

\def\gG{{\mathcal{G}}}

\newcommand{\R}{\mathbb{R}}

\def\mPhi{{\bm{\Phi}}}
\def\mPsi{{\bm{\Psi}}}
\def\mTheta{{\bm{\Theta}}}

\newcommand\citep[1]{\cite{#1}}
\newcommand\citet[1]{\cite{#1}}

\def\BibTeX{{\rm B\kern-.05em{\sc i\kern-.025em b}\kern-.08em
    T\kern-.1667em\lower.7ex\hbox{E}\kern-.125emX}}
\usepackage{balance}

\begin{document}
\title{Learnable Filters for Geometric Scattering Modules} %

\author{Alexander~Tong\textsuperscript{*}, %
        Frederik Wenkel\textsuperscript{*},
        Dhananjay Bhaskar,
        Kincaid Macdonald,
        Jackson Grady,
        Michael Perlmutter,
        Smita Krishnaswamy\textsuperscript{†}, %
        and~Guy Wolf\textsuperscript{†}%
\thanks{\textsuperscript{*} denotes equal contribution and \textsuperscript{†} equal senior author contribution. A. Tong is with the Dept. of Computer Science and Operations Research and F. Wenkel and G. Wolf are with the Dept. of Mathematics and Statistics at Universit\'{e} de Montr\'{e}al. M. Perlmutter is with the Dept. of Mathematics at University of California Los Angeles. K. Macdonald is with the Dept. of Mathematics, J. Grady is with the Dept. of Computer Science, D. Bhaskar is with the Dept. of Genetics, and S. Krishnaswamy is with the Depts. of Genetics and Computer Science at Yale University, New Haven, CT, USA. A. Tong, F. Wenkel, and G. Wolf are also affiliated with Mila -- the Quebec AI Institute, Montreal, QC, Canada.}%
\thanks{ Correspondence to: Smita Krishnaswamy~\textless{}smita.krishnaswamy@yale.edu\textgreater{}, 333 Ceder Street, New Haven, CT 06520, United States}%
}%

\maketitle
\begin{abstract}
We propose a new graph neural network (GNN) module, based on relaxations of recently proposed geometric scattering transforms, which consist of a cascade of graph wavelet filters. Our learnable geometric scattering (LEGS) module enables adaptive tuning of the wavelets to encourage band-pass features to emerge in learned representations. The incorporation of our LEGS-module in GNNs enables the learning of longer-range graph relations compared to many popular GNNs, which often rely on encoding graph structure via smoothness or similarity between neighbors. Further, its wavelet priors result in simplified architectures with significantly fewer learned parameters compared to competing GNNs. We demonstrate the predictive performance of LEGS-based networks on graph classification benchmarks, as well as the descriptive quality of their learned features in biochemical graph data exploration tasks. Our results show that LEGS-based networks match or outperforms popular GNNs, as well as the original geometric scattering construction, on many datasets, in particular in biochemical domains, while retaining certain mathematical properties of handcrafted (non-learned) geometric scattering.

\end{abstract}


\section{Introduction}\label{sec: introduction}
\IEEEPARstart{G}{eometric} deep learning has recently emerged as an increasingly prominent branch of deep learning~\citep{bronstein_geometric_2017}. At the core of geometric deep learning is the use of graph neural networks (GNNs) in general, and graph convolutional networks (GCNs) in particular, which ensure neuron activations follow the geometric organization of input data by propagating information across graph neighborhoods~\citep{bruna_spectral_2014, defferrard_convolutional_2016, kipf_semi-supervised_2016, hamilton_inductive_2017, xu_how_2019, abu-el-haija_mixhop_2019-1}. However, recent work has shown the difficulty in generalizing these methods to more complex structures, identifying common problems and phrasing them in terms of so-called oversmoothing~\citep{li2018deeper}, underreaching~\citep{barcelo2020logical} and oversquashing~\citep{alon2020bottleneck}.

Using graph signal processing terminology from \citet{kipf_semi-supervised_2016}, these issues can be partly attributed to the limited construction of convolutional filters in many commonly used GCN architectures. Inspired by the filters learned in convolutional neural networks, GCNs consider node features as graph signals and aim to aggregate information from neighboring nodes. For example, \citet{kipf_semi-supervised_2016} presented a typical implementation of a GCN with a cascade of averaging (essentially low pass) filters. We note that more general variations of GCN architectures exist~\citep{defferrard_convolutional_2016,hamilton_inductive_2017, xu_how_2019}, which are capable of representing other filters, but as investigated in~\citet{alon2020bottleneck}, they often have difficulty in learning long-range connections. 

Recently, an alternative approach was presented to provide deep geometric representation learning by generalizing Mallat's scattering transform \citep{mallat_group_2012}, originally proposed to provide a mathematical framework for understanding convolutional neural networks, to graphs~\citep{gao2019geometric, gama2019diffusion, zou_graph_2019} and manifolds~\citep{perlmutter_geometric_2018,mcewen2021scattering,chew2022}. 
The geometric scattering transform can represent nodes or graphs based on multi-scale diffusions, and differences between scales of diffusions of graph signals (i.e., node features). 
Similar to traditional scattering, which can be seen as a convolutional network with non-learned wavelet filters, geometric scattering is defined as a GNN with handcrafted graph filters, constructed with diffusion wavelets over the input graph~\citep{coifman_diffusion_2006-1}, which are then cascaded with pointwise absolute-value nonlinearities.
The efficacy of geometric scattering features in graph processing tasks was demonstrated in~\citet{gao2019geometric}, with both supervised learning and data exploration applications. Moreover, their handcrafted design enables rigorous study of their properties, such as stability to deformations and perturbations, and provides a clear understanding of the information extracted by them, which by design (e.g., the cascaded band-pass filters) goes beyond low frequencies to consider richer notions of regularity~\citep{gama_stability_2019, perlmutter2019understanding}. 

However, while geometric scattering transforms provide effective universal feature extractors, their handcrafted design does not allow the automatic task-driven representation learning that is so successful in traditional GNNs and neural networks in general. Here, we combine both frameworks by incorporating richer multi-frequency band features from geometric scattering into GNNs, while allowing them to be flexible and trainable. We introduce the geometric scattering module, which can be used within a larger neural network.
We call this a {\em learnable geometric scattering (LEGS) module} and show it inherits properties from the scattering transform while allowing the scales of the diffusion to be learned. Moreover, we show that our framework is differentiable, allowing for backpropagation through it in a standard reverse mode auto differentiation library. 

The benefits of our construction over standard GNNs, as well as pure geometric scattering, are discussed and demonstrated on graph classification and regression tasks in Sec.~\ref{sect_results}. In particular, we find that our network maintains the robustness to small training sets present in geometric scattering while improving classification on biological graph classification and regression tasks, in particular, in tasks where the graphs have a large diameter relative to their size, learnable scattering features improve performance over competing methods. We show that our construction performs better on tasks that require whole-graph representations with an emphasis on biochemical molecular graphs, where relatively large diameters and non-planar structures usually limit the effectiveness of traditional GNNs. We also show that our network maintains performance in social network and other node classification tasks where state-of-the-art GNNs perform well.

A previous short version of this work appeared in the IEEE Workshop on Machine Learning and Signal Processing 2021~\cite{tong_data-driven_2021}. We expand on that work first by incorporating additional theory including Theorem \ref{thm: random scales} which generalizes existing theory for nonexpansive scattering operators. Furthermore, we add additional experiments on molecular data, as well as ablation studies on both amount of training data and ensembling with other models.

The remainder of this paper is organized as follows. In Section~\ref{sec:related}, we review related work on graph scattering and literature on the challenges of modern GNNs. In Section~\ref{sect_geometric scattering}, we review some of the concepts of geometric scattering. In Section~\ref{sect_theory}, we present expanded theory on geometric scattering with task-driven tuning. This theory establishes that our LEGS module retains the theoretical properties of scattering while increasing expressiveness. In Section~\ref{sect_architecture}, we present the architecture and implementation details of the LEGS module. We examine the empirical performance of LEGS architectures in Section~\ref{sect_results} and conclude in Section~\ref{sec:conclusion}.

\section{Related Work}\label{sec:related}

A widely discussed challenge for many modern GNN approaches is so-called oversmoothing~\citep{li2018deeper,balcilar2020analyzing}. This is a result of the classic message passing in GNNs that is based on cascades of local node feature aggregations over node neighborhoods.
This increasingly smooths graph signals, which in turn renders the graph nodes undistinguishable. From a spectral point of view, this phenomenon is due to most GNN filters being low-pass filters~\citep{nt2019revisiting} that mostly preserve the low-frequency spectrum. A related phenomenon is so-called underreaching~\citep{barcelo2020logical}, which is a result of the limited spatial support of most GNN architectures. Most models can only relate information from nodes within a distance equal to the number of layers. Hence, they cannot represent long-range interactions as the before mentioned oversmoothing typically prohibits the design of truly ``deep'' GNN architectures. Lastly, oversquashing~\citep{alon2020bottleneck} is yet another consequence of typical message passing. As the number of nodes in the receptive field of each node grows exponentially in the number of GNN layers, a huge amount of information needs to be compressed into a vector of fixed size. This makes it difficult to represent meaningful relationships between nodes, as the contribution of single nodes becomes marginal. 

We note that efforts to improve the capture of long-range connections in graph representation learning have recently yielded several spectral approaches based on using the Lancoz algorithm to approximate graph spectra~\cite{liao_lanczosnet_2019}, or based on learning in block Krylov subspaces~\cite{luan_break_2019}. Such methods are complementary to the work presented here, in that their spectral approximation can also be applied in the computation of geometric scattering when considering very long range scales (e.g., via spectral formulation of graph wavelet filters). However, we find that such approximations are not necessary in the datasets considered here and in other recent work focusing on whole-graph tasks, where direct computation of polynomials of the Laplacian is sufficient. Furthermore, recent attempts have also considered ensemble approaches with hybrid architectures that combine GCN and scattering channels~\citep{min2020scattering}, albeit primarily focused on node-level tasks, considered on a single graph at a time, rather than whole-graph tasks considered here on datasets comparing multiple graphs. Such ensemble approaches are also complimentary to the proposed approach in that hybrid architectures can also be applied in conjunction with the proposed LEGS module here as we demonstrate in Sec.~\ref{sect_results}. 

\section{Preliminaries: Geometric Scattering}\label{sect_geometric scattering} %

Let $\gG = (V,E,w)$ be a weighted graph with $V\coloneqq \{v_1,\dots,v_n\}$ the set of nodes, $E\subset \{\{v_i, v_j\}\in V\times V , i\neq j\}$ the set of (undirected) edges and $w : E \to (0,\infty)$ assigning (positive) edge weights to the graph edges.
We define a \textit{graph signal} as a function $x: V \rightarrow \R$ on the nodes of $\gG$ and aggregate them in a signal vector $\vx\in \R^n$ with the $i^{th}$ entry being $x(v_i)$.
We define the \textit{weighted adjacency matrix} $\mW\in\R^{n\times n}$ of $\gG$ as
    $\mW[v_i,v_j] \coloneqq
    w(v_i,v_j) \text{ if } \{v_i,v_j\}\in E, 
    \text{ and } 0 \text{ otherwise}$
and the \textit{degree matrix} $\mD\in\R^{n\times n}$ of $\gG$ as $\mD\coloneqq \diag(d_1,\dots, d_n)$ with $d_i\coloneqq \deg(v_i)\coloneqq \sum_{j=1}^n \mW[v_i,v_j]$ the \textit{degree} of node $v_i$.

The geometric scattering transform~\citep{gao2019geometric} consists of a cascade of graph filters constructed from a left stochastic diffusion matrix $\mP \coloneqq \frac{1}{2} \big( \mI_n + \mW \mD^{-1} \big)$, which corresponds to transition probabilities of a lazy random walk Markov process. The laziness of the process signifies that at each step it has equal probability of staying at the current node or transitioning to a neighbor.
Scattering filters are defined via graph-wavelet matrices $\mPsi_j\in\R^{n\times n}$ of order $j\in\N_0$, as
\begin{align}\label{eq_wavelet matrix}
    \mPsi_0 &\coloneqq \mI_n - \mP, \nonumber \\
    \mPsi_j &\coloneqq \mP^{2^{j-1}} - \mP^{2^j} = \mP^{2^{j-1}} \big( \mI_n - \mP^{2^{j-1}} \big), \quad j\geq 1.
\end{align}
These diffusion wavelet operators partition the frequency spectrum into dyadic frequency bands, which are then organized into a full wavelet filter bank $\mathcal{W}_J\coloneqq\{\mPsi_j, \mPhi_J\}_{0\leq j\leq J}$, where $\mPhi_J\coloneqq \mP^{2^J}$ is a pure low-pass filter, similar to the one used in GCNs. It is easy to verify that the resulting wavelet transform is invertible, since a simple sum of filter matrices in $\mathcal{W}_J$ yields the identity. Moreover, as discussed in~\citet{perlmutter2019understanding}, this filter bank forms a nonexpansive frame, which provides energy preservation guarantees, as well as stability to perturbations, and can be generalized to a wider family of constructions that encompasses the variations of scattering transforms on graphs from such as those considered in~\citet{gama2019diffusion}.

Given the wavelet filter bank $\mathcal{W}_J$, node-level scattering features are computed by stacking cascades of bandpass filters and element-wise absolute value nonlinearities to form 
\begin{equation}\label{eq_scattering (node) features}
     \mU_p \vx \coloneqq \mPsi_{j_m} \vert \mPsi_{j_{m-1}} \dots \vert \mPsi_{j_2} \vert \mPsi_{j_1}\vx\vert \vert \dots \vert,
\end{equation}
parameterized by the scattering path $p \coloneqq (j_1, \dots, j_m)\in \cup_{m \in \N} \N_0^{m}$ that determines the filter scales of each wavelet. Whole-graph representations are obtained by aggregating node-level features via statistical moments over the nodes of the graph~\citep{gao2019geometric}, which yields the geometric scattering moments
\begin{equation}\label{eq_scattering (graph) featrues}
    \mS_{p,q} \vx \coloneqq \sum_{i=1}^n \vert \mU_p \vx [v_i] \vert^q,
\end{equation}
indexed by the scattering path $p$ and moment order $q$. Finally, we note that Theorem 3.6 of \cite{perlmutter2019understanding} shows that $\mU_p$ is equivariant to permutations of the nodes. Therefore, it follows that the graph-level scattering transform $\mS_{p,q}$ is node-permutation invariant since it is defined via global summation.

\section{Adaptive Geom. Scattering Relaxation}\label{sect_theory} %

The geometric scattering construction, described in Sec.~\ref{sect_geometric scattering}, can be seen as a particular GNN architecture with handcrafted layers, rather than learned ones. This provides a solid mathematical framework for understanding the encoding of geometric information in GNNs~\citep{perlmutter2019understanding}, while also providing effective unsupervised graph representation learning for data exploration, which also has advantages in supervised learning tasks~\citep{gao2019geometric}. 

Both
\citet{perlmutter2019understanding} and \citet{gao2019geometric} used dyadic scales in Eq.~\ref{eq_wavelet matrix}, a choice inspired by the Euclidean scattering transform \cite{mallat_group_2012}. Below in Theorem \ref{thm:frame}, we will show that these dyadic scales may be replaced by \emph{any} increasing sequence of scales and the resulting wavelets will still form a nonexpansive frame. Later in Section \ref{sect_architecture}, Theorem \ref{thm: random scales} will consider scales which are learned from data and show that the learned filter bank forms a nonexpansive operator under mild assumptions. This allows us to obtain a flexible model with similar guarantees to the model considered in  \cite{gao2019geometric},
but which is amenable to task-driven tuning provided by end-to-end GNN training. %

Given an increasing sequence of integer diffusion time scales $0 < t_1 < \cdots < t_J$, we replace the wavelets considered in Eq.~\ref{eq_wavelet matrix} with the  generalized filter bank $\mathcal{W}_J^\prime \coloneqq\{\mPsi_j^\prime, \mPhi_J^\prime\}_{j=0}^{J-1}$, where
\begin{align}\label{eq_adaptive wavelet matrices}
    \mPsi_0^\prime &\coloneqq  \mI_n - \mP^{t_1}, \quad \mPhi_J^\prime \coloneqq \mP^{t_J},%
     \\
    \mPsi_j^\prime &\coloneqq \mP^{t_j} - \mP^{t_{j+1}}, \quad 1 \leq j \leq J-1. \nonumber
\end{align}
To study these wavelets, it will be convenient to consider a weighted inner product space $L^2(\gG,\mM)$ of graph signals with inner product
$
    \langle \vx, \vy \rangle \coloneqq \langle \mD^{-\frac{1}{2}}\vx, \mD^{-\frac{1}{2}}\vy \rangle
$
and induced norm
$
    {\Vert \vx \Vert}_{\mD^{-\frac{1}{2}}}^2 = {\Vert \mD^{-\frac{1}{2}} \vx \Vert}_2 = \sum_{i=1}^n\frac{\vx[i]^2}{d_i}
$
for $\vx,\vy\in L^2(\gG,\mD^{-\frac{1}{2}})$.

The following theorem shows that  $\mathcal{W}_J^\prime$ is a nonexpansive frame, similar to the result shown for dyadic scales in~\citet{perlmutter2019understanding}.

\begin{thm}\label{thm:frame}
There exists a constant $C > 0$ that only depends on $t_1$ and $t_J$ such that for all $\vx\in L^2(\gG,\mD^{-1/2})$,
$$
    C \Vert \vx \Vert_{\mD^{-\frac{1}{2}}}^2 \leqslant \Vert \mPhi_J^\prime \vx \Vert_{\mD^{-\frac{1}{2}}}^2 + \sum_{j=0}^J\Vert \mPsi_j^\prime \vx \Vert_{\mD^{-\frac{1}{2}}}^2 \leqslant \Vert \vx \Vert_{\mD^{-\frac{1}{2}}}^2 ,
$$
where the norm is the one induced by the space $L^2(\gG,\mD^{-1/2})$.
\end{thm}

\begin{proof}
Note that $\mP$ has a symmetric conjugate %
$\mM\coloneqq\mD^{-1/2} \mP\mD^{1/2}$ with eigendecomposition $\mM = \mQ \mLambda \mQ^T$ for orthogonal $\mQ$.
Given this decomposition, we can write
\begin{align}
    \mPhi_J^\prime &= \mD^{1/2} \mQ \mLambda^{t_J} \mQ^T \mD^{-1/2}, \nonumber \\
    \mPsi_j^\prime &= \mD^{1/2} \mQ (\mLambda^{t_j} - \mLambda^{t_{j+1}}) \mQ^T \mD^{-1/2}, \quad 0 \leq j \leq J-1, \nonumber
\end{align}
where we set $t_0 = 0$ to simplify notations. Therefore, we have
\begin{equation*}
\|\mPhi_J^\prime \vx \|_{\mD^{-1/2}}^2 = \langle \mPhi_J^\prime \vx, \mPhi_J^\prime \vx \rangle_{\mD^{-1/2}} = \|\mLambda^{t_J} \mQ^T \mD^{-1/2} \vx\|_2^2.
\end{equation*}
If we consider a change of variable to $\vy = \mQ^T \mD^{-1/2} \vx$, we have $\|\vx\|_{\mD^{-1/2}}^2 = \|\mD^{-1/2}\vx\|_2^2 = \|\vy\|_2^2$, while $\|\mPhi_J^\prime \vx \|_{\mD^{-1/2}}^2 = \|\Lambda^{t_J}\vy\|_2^2$. Similarly, we can also reformulate the operations of the other filters in terms of diagonal matrices applied to $\vy$ as $\|\mPsi_j^\prime \vx \|_{\mD^{-1/2}}^2 = \|(\Lambda^{t_j} - \Lambda^{t_{j+1}})\vy\|_2^2$. 

Given these reformulations, we can now write
\begin{multline*}
    \|\Lambda^{t_J}\vy\|_2^2 + \sum_{j=0}^{J-1}\|(\Lambda^{t_j} - \Lambda^{t_{j+1}})\vy\|_2^2 = \\
    \sum_{i=1}^n \vy_i^2 \cdot \left(\lambda_i^{2 t_J} + \sum\nolimits_{j=0}^{J-1}(\lambda_i^{t_j} - \lambda_i^{t_{j+1}})^2\right).
\end{multline*}
Since $0 \leq \lambda_i \leq 1$ and $0 = t_0 < t_1 < \cdots < t_J$, we have
$$
\lambda_i^{2 t_J} + \sum_{j=0}^{J-1}(\lambda_i^{t_j} - \lambda_i^{t_{j+1}})^2 \leq \left(\lambda_i^{t_J} + \sum_{j=0}^{J-1}\lambda_i^{t_j} - \lambda_i^{t_{j+1}}\right)^2
\leq 1,
$$
which yields the upper bound $\|\Lambda^{t_J}\vy\|_2^2 + \sum_{j=0}^{J-1}\|(\Lambda^{t_j} - \Lambda^{t_{j+1}})\vy\|_2^2 \leq \|\vy\|_2^2$.
On the other hand, since $t_1 > 0 = t_0$,
$$
\lambda_i^{2 t_J} + \sum_{j=0}^{J-1}(\lambda_i^{t_j} - \lambda_i^{t_{j+1}})^2 \geq \lambda_i^{2 t_J} + (1 - \lambda_i^{t_1})^2,
$$ 
and thus, setting $C \coloneqq \min_{0 \leq \xi \leq 1} (\xi^{2 t_J} + (1 - \xi^{t_1})^2) > 0$,
we get the lower bound $\|\Lambda^{t_J}\vy\|_2^2 + \sum_{j=0}^{J-1}\|(\Lambda^{t_j} - \Lambda^{t_{j+1}})\vy\|_2^2 \geq C \|\vy\|_2^2$. Applying the reverse change of variable to $\vx$ and $L^2(\gG,\mD^{-1/2})$ yields the result of the theorem.
\end{proof}

Theorem \ref{thm:frame} shows that the wavelet transform is injective and stable to additive noise.
Our next theorem %
shows that it is permutation equivariant at the node-level and permutation invariant at the graph level. This guarantees that the extracted  features  encode the intrinsic  geometry of the graph rather than a priori indexation.
\begin{thm}\label{thm:permutation}
Let $\mU_p^\prime$ and $\mS_{p,q}^\prime$ be defined as in Eq.~\ref{eq_scattering (node) features}-\ref{eq_scattering (graph) featrues}, with the filters from $\mathcal{W}_J^\prime$ with an arbitrary configuration $0 < t_1 < \cdots < t_J$ in place of $\mathcal{W}_J$. Then, for any permutation $\Pi$ over the nodes and any graph signal $\vx\in L^2(\gG,\mD^{-1/2})$, we have $
\mU_p^\prime \Pi \vx = \Pi \mU_p^\prime \vx$ and $\mS_{p,q}^\prime \Pi \vx = \mS_{p,q}^\prime \vx$, for $p \in \cup_{m \in \N} \N_0^{m}, q \in \N$, where geometric scattering implicitly considers the node ordering supporting its input signal.
\end{thm}

\begin{proof}
For any permutation $\Pi$, we let $\overline{\gG} = \Pi(\gG)$ be the graph obtained by permuting the vertices of $\gG$ with $\Pi$. The corresponding permutation operation on a graph signal $\vx \in L^2(\gG,\mD^{-1/2})$ gives a signal $\Pi \vx \in L^2(\overline{\gG}, \mD^{-1/2})$, which we implicitly considered in the statement of the theorem, without specifying these notations for simplicity. Rewriting the statement of the theorem more rigorously with the introduced notations, we aim to show that $\overline{\mU}_p^\prime \Pi \vx = \Pi \mU_p^\prime \vx$ and $\overline{\mS}_{p,q}^\prime \Pi \vx = \mS_{p,q}^\prime \vx$ under suitable conditions, where the operation $\mU_p^\prime$ from $\gG$ on the permuted graph $\overline{\gG}$ is denoted here by $\overline{\mU}_p^\prime$ and likewise for $\mS_{p,q}^\prime$ we have $\overline{\mS}_{p,q}^\prime$. 

We start by showing $\mU_p^\prime$ is permutation equivariant. First, we notice that for any $\Psi_j$, $0 < j < J$ we have that $\overline{\Psi}_j \Pi \vx = \Pi \Psi_j \vx$, as for $1 \le j \le J - 1$,
$$
    \overline{\mPsi}_j \Pi \vx 
    = (\Pi \mP^{t_j} \Pi^T - \Pi\mP^{t_{j+1}}\Pi^T) \Pi \vx
    = \Pi \mPsi_j \vx,
$$
with similar reasoning for $j\in \{0, J\}$. Note that the element-wise absolute value yields $\vert \Pi \vx \vert = \Pi \vert \vx \vert$ for any permutation matrix $\Pi$. These two observations inductively yield 
\begin{align*}
    \overline{\mU}_p^\prime \Pi\vx =& \overline{\mPsi}_{j_m}^\prime \vert \overline{\mPsi}_{j_{m-1}}^\prime \dots \vert \overline{\mPsi}_{j_2}^\prime \vert \overline{\mPsi}_{j_1}^\prime \Pi\vx\vert \vert \dots \vert \\
    =&  \overline{\mPsi}_{j_m}^\prime \vert \overline{\mPsi}_{j_{m-1}}^\prime \dots \vert \overline{\mPsi}_{j_2}^\prime \Pi \vert \mPsi_{j_1}^\prime \vx\vert \vert \dots \vert =\dots= \Pi \mU_p^\prime \vx. 
\end{align*}
To show $\mS_{p,q}^\prime$ is permutation invariant, first notice that for any statistical moment $q > 0$, we have  $\vert \Pi \vx \vert^q = \Pi \vert \vx \vert^q$ and further, as sums are commutative, $\sum_j (\Pi \vx)_j = \sum_j \vx_j$. We then have
$$
    \overline{\mS}_{p,q}^\prime \Pi \vx = \sum_{i=1}^n \vert \overline{\mU}_p^\prime \Pi \vx [v_i] \vert^q = \sum_{i=1}^n \vert \Pi \mU_p^\prime \vx [v_i] \vert^q
    = \mS_{p,q}^\prime \vx,
$$
which
completes the proof of the theorem.
\end{proof}

We note that the results in Theorems~\ref{thm:frame}-\ref{thm:permutation} and their proofs closely follow the theoretical framework proposed by~\citet{perlmutter2019understanding}. We carefully account here for the relaxed learned configuration, which replaces the original handcrafted one there.

\begin{figure*}[ht]
    \begin{center}
    \includegraphics[width=1\linewidth]{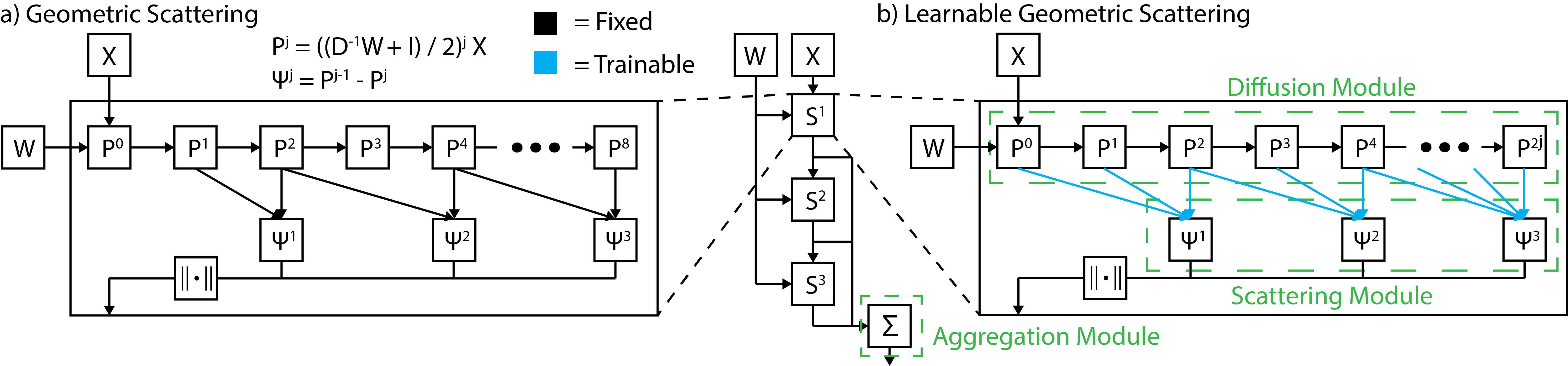}
    \caption{LEGS module learns to select the appropriate scattering scales from the data.}
    \end{center}
    \label{fig:arch}
\end{figure*}
\section{A Learnable Geom. Scattering Module}\label{sect_architecture} %

In this section, we show how the generalized geometric scattering construction presented in Sec.~\ref{sect_theory} can be implemented in a data-driven way via a backpropagation-trainable module. Throughout this section, we consider an input graph signal $\vx\in\R^n$ or, equivalently, a collection of graph signals $\mX\in\R^{n\times N_{\ell-1}}$. The forward propagation of these signals can be divided into three major submodules. First, a {\em diffusion submodule} implements the Markov process that forms the basis of the filter bank and transform. Then, a {\em scattering submodule} implements the filters and the corresponding cascade, while allowing the learning of the scales $t_1,\ldots,t_J$. Finally, the {\em aggregation module} collects the extracted features to provide a graph and produces the task-dependent output.

\vspace{2pt}\noindent\textbf{The diffusion submodule.} We build a set of $m\in\N$ subsequent diffusion steps of the signal $\vx$ by iteratively multiplying the diffusion matrix $\mP$ to the left of the signal, resulting in
$
    \left[ \mP \vx, \mP^2 \vx, \dots, \mP^m \vx \right].
$
Since $\mP$ is often sparse, for efficiency reasons these filter responses are implemented via an RNN structure consisting of $m$ RNN modules. Each module propagates the incoming hidden state $\vh_{t-1}, t = 1, \ldots, m$ with $\mP$ with the readout $\vo_t$ equal to the produced hidden state,
$
    \vh_t \coloneqq \mP \vh_{t-1}, \quad \vo_t \coloneqq \vh_t.
$

\vspace{2pt}\noindent\textbf{The scattering submodule.} Next, we consider the selection of $J\leq m$ diffusion scales for the flexible filter bank construction with wavelets defined according to~Eq.~\ref{eq_wavelet matrix relaxed}. We found this was the most influential part of the architecture. We experimented with methods of increasing flexibility:
\begin{enumerate*}
    \item Selection of $\{t_j\}_{j=1}^{J-1}$ as dyadic scales (as in Sec.~\ref{sect_geometric scattering} and Eq.~\ref{eq_wavelet matrix}), fixed for all datasets (LEGS-FIXED); and
    \item Selection of each $t_j$ using softmax and sorting by $j$, learnable per model (LEGS-FCN and LEGS-RBF, depending on output layer explained below).
\end{enumerate*}

\noindent For the scale selection, we use a selection matrix $\mF\in\R^{J\times m}$, where each row $\mF_{(j,\cdot)}, j = 1,\ldots,J$ is dedicated to identifying the diffusion scale of the wavelet $\mP^{t_j}$ via a one-hot encoding. This is achieved by setting $\mF
    \coloneqq \sigma(\mTheta) = [\sigma(\vtheta_1), \sigma(\vtheta_2), \ldots, \sigma(\vtheta_J)]^T,$
where $\vtheta_j \in\R^m$ constitute the rows of the trainable weight matrix $\mTheta$, and $\sigma$ is the softmax function. While this construction may not strictly guarantee an exact one-hot encoding, we assume that the softmax activations yield a sufficient approximation. Further, without loss of generality, we assume that the rows of $\mF$ are ordered according to the position of the leading ``one'' activated in every row. In practice, this can be easily enforced by reordering the rows. We now construct the filter bank $\widetilde{\mathcal{W}}_{\mF} \coloneqq\{\widetilde\mPsi_{j}, \widetilde\mPhi_{J}\}_{j=0}^{J-1}$ with the filters
\begin{align}\label{eq_wavelet matrix relaxed}
    \widetilde{\mPsi}_{0} \vx &= \vx - \sum_{t=1}^m \mF_{(1,t)} \mP^t \vx, \quad \widetilde{\mPhi}_{J} \vx = \sum_{t=1}^m \mF_{(J,t)} \mP^t \vx, \\
    \widetilde{\mPsi}_{j} \vx &= \sum_{t=1}^m \left[\mF_{(j,t)} \mP^t \vx - \mF_{(j+1,t)} \mP^t \vx \right], \quad 1 \leq j \leq J-1, \nonumber
\end{align}
matching and implementing the construction of $\mathcal{W}_J^\prime$ (Eq.~\ref{eq_adaptive wavelet matrices}).
The following theorem shows that $\widetilde{\mathcal{W}}_{\mF} \coloneqq\{\widetilde\mPsi_{j}, \widetilde\mPhi_{J}\}_{j=0}^{J-1}$ is a nonexpansive operator under the assumption that the rows of $\mathbf{F}$  have disjoint support. Since softmax only leads to approximately sparse rows, this condition will in general only hold approximately. However, it can be easily achieved via post-processing.
\begin{thm}\label{thm: random scales}
Suppose that the rows of $\mathbf{F}$ have disjoint support and that every element of $\text{supp}(\mathbf{F}_{j,\cdot})$ is less than every element of $\text{supp}(\mathbf{F}_{j+1,\cdot})$ for all $1\leq j\leq J-1$. Then $\widetilde{\mathbf{W}}_\mathbf{F}$ is a nonexpansive operator, i.e.
\begin{equation*}
    \sum_{j=0}^J \|\mathbf{\widetilde{\Psi}}_j{\mathbf{x}}\Vert_{\mD^{-\frac{1}{2}}}^2+\|\mathbf{\widetilde{\Phi}}_J \mathbf{x}\Vert_{\mD^{-\frac{1}{2}}}^2\leq \|\mathbf{x}\Vert_{\mD^{-\frac{1}{2}}}^2
\end{equation*}
\end{thm}
The proof of Theorem \ref{thm: random scales} relies on applying Jensen's inequality and then Theorem \ref{thm:frame}. We provide full details in the Section \ref{sec: auxilliary proofs} of the appendix. %

\vspace{2pt}\noindent\textbf{The aggregation submodule.}
While many approaches may be applied to aggregate node-level features into graph-level features such as max, mean, sum pooling, or the more powerful TopK~\citep{gao_graph_2019-1} and attention pooling~\citep{velickovic_graph_2018}, we follow the statistical-moment aggregation explained in Secs.~\ref{sect_geometric scattering}-\ref{sect_theory} motivated by \citet{gao2019geometric,perlmutter2019understanding} and leave exploration of other pooling methods to future work.

\subsection{Incorporating LEGS into a larger neural network} 

As shown in~\citet{gao2019geometric} on graph classification, this aggregation works particularly well in conjunction with support vector machines (SVMs) based on the radial basis function (RBF) kernel. Here, we consider two configurations for the task-dependent output layer of the network, either using two fully connected layers after the learnable scattering layers, which we denote LEGS-FCN, or using a modified RBF network~\citep{broomhead_radial_1988}, which we denote LEGS-RBF, to produce the final classification. 

The latter configuration more accurately processes scattering features as shown in Table~\ref{tab:full_table}. Our RBF network works by first initializing a fixed number of movable anchor points. Then, for every point, new features are calculated based on the radial distances to these anchor points. In previous work on radial basis networks these anchor points were initialized independent of the data. We found that this led to training issues if the range of the data was not similar to the initialization of the centers. Instead, we first use a batch normalization layer to constrain the scale of the features and then pick anchors randomly from the initial features of the first pass through our data. This gives an RBF-kernel network with anchors that are always in the range of the data. Our RBF layer is then $\text{RBF}(\vx) = \phi(\| \text{BatchNorm}(\vx) - \vc\|)$ with $\phi(\vx) = e^{-\|\vx\|^2}$.

\subsection{Backpropagation through the LEGS module}
The LEGS module is fully suitable for incorporation in any neural network architecture and can be backpropagated through. To show this we write the partial derivatives with respect to the LEGS module parameters $\mTheta$ and $\mW$ here.
The gradients of the filter $\widetilde{\mPsi}_j, j\geq 1,$
with respect to scale weights $\theta_k$, $k = 1, \ldots m$, where $\sigma$ is the softmax function, can be written as
\begin{align}
    \frac{\partial \widetilde{\mPsi}_j}{\partial \mTheta_{k, t}} = 
    \begin{cases}
    \mP^t \sigma(\vtheta_j)_t(1-\sigma(\vtheta_j))_t & \text{ if } k = j, \\
    \mP^t \sigma(\vtheta_j)_t \sigma(\vtheta_k)_t & \text{ if } k = j + 1, \\
    \boldsymbol{0}_{n\times n} & \text{ else.}
    \end{cases}
\end{align}
Finally, the gradient with respect to the adjacency matrix entry $\mW_{ab}$, where we denote $\mJ^{ab}$ the matrix with $\mJ_{(a,b)}^{ab}=1$ and all other entries zero, can be written as
\begin{align}
    \frac{\partial \widetilde{\mPsi}_j}{\partial \mW_{ab}} = \sum_{t=1}^m \biggl [&(\mF_{j, t} - \mF_{j+1,t}) \times \nonumber \\ &\sum_{k=1}^t \mP^{k-1} (\mJ^{ab} \mD^{-1}) \mP^{t - k} \biggr ].
\end{align}
Gradients of filters $\widetilde{\mPhi}_J$ and $\widetilde{\mPsi}_0$, which are simple modifications of the partial derivatives of $\widetilde{\mPsi}_j$ and derivations of these gradients can be found in Sec.~\ref{sec:supp:grad} of the appendix.

\section{Empirical Results}\label{sect_results}
We investigate the LEGS module on whole graph classification and graph regression tasks that arise in a variety of contexts, with emphasis on the more complex biochemical datasets. Unlike other types of data, biochemical graphs do not exhibit the small-world structure of social graphs and may have large diameters relative to their size. Further, the connectivity patterns of biomolecules are very irregular due to 3D folding and long-range connections, and thus ordinary local node aggregation methods may miss such connectivity differences.

\begin{table}[ht]
    \caption{Dataset statistics, diameter, nodes, edges, clustering coefficient (CC) averaged over all graphs. Split into bio-chemical and social network types.}\label{tab:dataset_stats}
\centering
\adjustbox{width=\linewidth}{
    \begin{tabular}{lrrrrrr}
\toprule
{} &  \# Graphs &  \# Classes &  Diameter &   Nodes &    Edges &  CC \\
\midrule
DD               &      1178 &          2 &     19.81 &  284.32 &   715.66 &          0.48 \\
ENZYMES          &       600 &          6 &     10.92 &   32.63 &    62.14 &          0.45 \\
MUTAG            &       188 &          2 &      8.22 &   17.93 &    19.79 &          0.00 \\
NCI1             &      4110 &          2 &     13.33 &   29.87 &    32.30 &          0.00 \\
NCI109           &      4127 &          2 &     13.14 &   29.68 &    32.13 &          0.00 \\
PROTEINS         &      1113 &          2 &     11.62 &   39.06 &    72.82 &          0.51 \\
PTC              &       344 &          2 &      7.52 &   14.29 &    14.69 &          0.01 \\
\midrule
COLLAB           &      5000 &          3 &      1.86 &   74.49 &  2457.22 &          0.89 \\
IMDB-B     &      1000 &          2 &      1.86 &   19.77 &    96.53 &          0.95 \\
IMDB-M       &      1500 &          3 &      1.47 &   13.00 &    65.94 &          0.97 \\
REDDIT-B    &      2000 &          2 &      8.59 &  429.63 &   497.75 &          0.05 \\
REDDIT-12K &     11929 &         11 &      9.53 &  391.41 &   456.89 &          0.03 \\
REDDIT-5K  &      4999 &          5 &     10.57 &  508.52 &   594.87 &          0.03 \\
\bottomrule
\end{tabular}}
\end{table}

\subsection{Whole Graph Classification}

We perform whole graph classification by using eccentricity (max distance of a node to other nodes) and clustering coefficient (percentage of links between the neighbors of the node compared to a clique) as node features as are used in \citet{gao2019geometric}. We compare against graph convolutional networks (GCN)~\citep{kipf_semi-supervised_2016}, GraphSAGE~\citep{hamilton_inductive_2017}, graph attention network (GAT)~\citep{velickovic_graph_2018}, graph isomorphism network (GIN)~\citep{xu_how_2019}, Snowball network~\citep{luan_break_2019}, and fixed geometric scattering with a support vector machine classifier (GS-SVM) as in \citet{gao2019geometric}, and a baseline which is a 2-layer neural network on the features averaged across nodes (disregarding graph structure). 

These comparisons are meant to inform when including learnable graph scattering features are helpful in extracting whole graph features. Specifically, we are interested in the types of graph datasets where existing graph neural network performance can be improved upon with scattering features. We evaluate these methods across 7 biochemical datasets and 6 social network datasets where the goal is to classify between two or more classes of compounds with hundreds to thousands of graphs and tens to hundreds of nodes (See Table~\ref{tab:dataset_stats}). For more specific information on individual datasets see Appendix~\ref{sup:dataset}. We use 10-fold cross validation on all models which is elaborated on in Section \ref{sup:training} of the supplement.

\begin{table*}[t]
\caption{Mean $\pm$ std.\ over 10 test sets on biochemical (top) and social network (bottom) datasets.}
 \label{tab:full_table}
\centering
\scalebox{0.8}{
        \begin{tabular}{llllllllll}
\toprule
{} &                   LEGS-RBF &                   LEGS-FCN &         LEGS-FIXED &                        GCN &                  GraphSAGE &               GAT &                        GIN &                     GS-SVM &                   Baseline \\
\midrule
DD               &           72.58 $\pm$ 3.35 &           72.07 $\pm$ 2.37 &   69.09 $\pm$ 4.82 &           67.82 $\pm$ 3.81 &           66.37 $\pm$ 4.45 &  68.50 $\pm$ 3.62 &           42.37 $\pm$ 4.32 &           72.66 $\pm$ 4.94 &  \textbf{75.98 $\pm$ 2.81} \\
ENZYMES          &           36.33 $\pm$ 4.50 &  \textbf{38.50 $\pm$ 8.18} &   32.33 $\pm$ 5.04 &           31.33 $\pm$ 6.89 &           15.83 $\pm$ 9.10 &  25.83 $\pm$ 4.73 &           36.83 $\pm$ 4.81 &           27.33 $\pm$ 5.10 &           20.50 $\pm$ 5.99 \\
MUTAG            &           33.51 $\pm$ 4.34 &           82.98 $\pm$ 9.85 &  81.84 $\pm$ 11.24 &           79.30 $\pm$ 9.66 &          81.43 $\pm$ 11.64 &  79.85 $\pm$ 9.44 &           83.57 $\pm$ 9.68 &  \textbf{85.09 $\pm$ 7.44} &           79.80 $\pm$ 9.92 \\
NCI1             &  \textbf{74.26 $\pm$ 1.53} &           70.83 $\pm$ 2.65 &   71.24 $\pm$ 1.63 &           60.80 $\pm$ 4.26 &           57.54 $\pm$ 3.33 &  62.19 $\pm$ 2.18 &           66.67 $\pm$ 2.90 &           69.68 $\pm$ 2.38 &           56.69 $\pm$ 3.07 \\
NCI109           &  \textbf{72.47 $\pm$ 2.11} &           70.17 $\pm$ 1.46 &   69.25 $\pm$ 1.75 &           61.30 $\pm$ 2.99 &           55.15 $\pm$ 2.58 &  61.28 $\pm$ 2.24 &           65.23 $\pm$ 1.82 &           68.55 $\pm$ 2.06 &           57.38 $\pm$ 2.20 \\
PROTEINS         &           70.89 $\pm$ 3.91 &           71.06 $\pm$ 3.17 &   67.30 $\pm$ 2.94 &           74.03 $\pm$ 3.20 &           71.87 $\pm$ 3.50 &  73.22 $\pm$ 3.55 &  \textbf{75.02 $\pm$ 4.55} &           70.98 $\pm$ 2.67 &           73.22 $\pm$ 3.76 \\
PTC              &  \textbf{57.26 $\pm$ 5.54} &           56.92 $\pm$ 9.36 &   54.31 $\pm$ 6.92 &          56.34 $\pm$ 10.29 &           55.22 $\pm$ 9.13 &  55.50 $\pm$ 6.90 &           55.82 $\pm$ 8.07 &           56.96 $\pm$ 7.09 &           56.71 $\pm$ 5.54 \\
\midrule
COLLAB           &           75.78 $\pm$ 1.95 &           75.40 $\pm$ 1.80 &   72.94 $\pm$ 1.70 &           73.80 $\pm$ 1.73 &  \textbf{76.12 $\pm$ 1.58} &  72.88 $\pm$ 2.06 &           62.98 $\pm$ 3.92 &           74.54 $\pm$ 2.32 &           64.76 $\pm$ 2.63 \\
IMDB-BINARY      &           64.90 $\pm$ 3.48 &           64.50 $\pm$ 3.50 &   64.30 $\pm$ 3.68 &           47.40 $\pm$ 6.24 &           46.40 $\pm$ 4.03 &  45.50 $\pm$ 3.14 &           64.20 $\pm$ 5.77 &  \textbf{66.70 $\pm$ 3.53} &           47.20 $\pm$ 5.67 \\
IMDB-MULTI       &           41.93 $\pm$ 3.01 &           40.13 $\pm$ 2.77 &   41.67 $\pm$ 3.19 &           39.33 $\pm$ 3.13 &           39.73 $\pm$ 3.45 &  39.73 $\pm$ 3.61 &           38.67 $\pm$ 3.93 &  \textbf{42.13 $\pm$ 2.53} &           39.53 $\pm$ 3.63 \\
REDDIT-BINARY    &  \textbf{86.10 $\pm$ 2.92} &           78.15 $\pm$ 5.42 &   85.00 $\pm$ 1.93 &           81.60 $\pm$ 2.32 &           73.40 $\pm$ 4.38 &  73.35 $\pm$ 2.27 &           71.40 $\pm$ 6.98 &           85.15 $\pm$ 2.78 &           69.30 $\pm$ 5.08 \\
REDDIT-MULTI-12K &           38.47 $\pm$ 1.07 &           38.46 $\pm$ 1.31 &   39.74 $\pm$ 1.31 &  \textbf{42.57 $\pm$ 0.90} &           32.17 $\pm$ 2.04 &  32.74 $\pm$ 0.75 &           24.45 $\pm$ 5.52 &           39.79 $\pm$ 1.11 &           22.07 $\pm$ 0.98 \\
REDDIT-MULTI-5K  &           47.83 $\pm$ 2.61 &           46.97 $\pm$ 3.06 &   47.17 $\pm$ 2.93 &  \textbf{52.79 $\pm$ 2.11} &           45.71 $\pm$ 2.88 &  44.03 $\pm$ 2.57 &           35.73 $\pm$ 8.35 &           48.79 $\pm$ 2.95 &           36.41 $\pm$ 1.80 \\
\bottomrule
\end{tabular}
    }
\end{table*}

\vspace{2pt}\noindent\textbf{LEGS outperforms on biochemical datasets.} Most work on graph neural networks has focused on social networks which have a well-studied structure. However, biochemical graphs that represent molecules and tend to be overall smaller and less connected than social networks (see Table~\ref{tab:dataset_stats}). In particular, we find that LEGS outperforms other methods by a significant margin on biochemical datasets with relatively small but high diameter graphs (NCI1, NCI109, ENZYMES, PTC), as shown in Table~\ref{tab:full_table}. On extremely small graphs we find that GS-SVM performs best, which is likely because other methods with more parameters can more easily overfit the data. We reason that the performance increase exhibited by LEGS module networks, and to a lesser extent GS-SVM, on these biochemical graphs is due the ability of geometric scattering to compute complex connectivity features via its multiscale diffusion wavelets. Thus, methods that rely on a scattering construction would in general perform better, with the flexibility and trainability of the LEGS module giving it an edge on most tasks.

\vspace{2pt}\noindent\textbf{LEGS performs well on social network datasets and considerably improves performance in ensemble models.} In Table~\ref{tab:full_table}, we see that on the social network datasets LEGS performs well. We note that one of the scattering style networks (either LEGS or GS-SVM) is either the best or second best on each dataset. On each of these datasets, LEGS-RBF (which uses a SVM with a radial basis function similar to GS-SVM) and GS-SVM are well within one standard deviation of each other.   %
If we also consider combining LEGS module features with GCN features the LEGS module performs the best on five out of six of the social network datasets. Across all datasets, an ensemble model considerably increases accuracy over GCN (see Table~\ref{tab:ensemble}). This underlines the capabilities of the LEGS module, not only as an isolated model, but also as a tool for powerful hybrid GNN architectures. Similar to \citet{min2020scattering}, this supports the claim that the LEGS module (due to geometric scattering) is sensitive to complementary regularity over graphs, compared to many traditional GNNs.
\begin{figure}[ht]
  \centering
  \includegraphics[width=.9\linewidth]{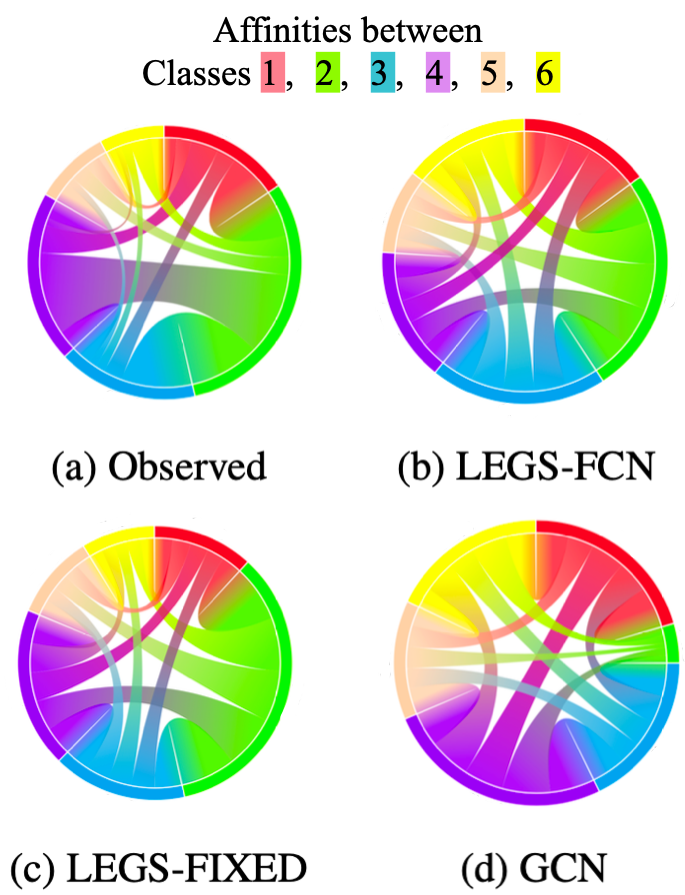}
  \label{fig:ribbon:sfig1}

\caption{Enzyme class exchange preferences empirically observed in \citet{cuesta_classification_2015}, and estimated from LEGS and GCN embeddings.}
\label{fig:ribbon}
\end{figure}

\vspace{2pt}\noindent\textbf{LEGS preserves enzyme exchange preferences while increasing performance.} One advantage of geometric scattering over other graph embedding techniques lies in the rich information present within the scattering feature space. 
This was demonstrated in \cite{gao2019geometric} where it was shown that the embeddings created through fixed geometric scattering can be used to accurately infer inter-graph relationships. Scattering features of enzyme graphs within the ENZYMES dataset~\citep{borgwardt_protein_2005} possessed sufficient global information to recreate the enzyme class exchange preferences observed empirically in \citet{cuesta_classification_2015}.
We demonstrate here that LEGS features retain similar descriptive capabilities, as shown in Figure~\ref{fig:ribbon}. Here we show chord diagrams where the chord size represents the exchange preference between enzyme classes, which is estimated as suggested in \cite{gao2019geometric}.
Our results here (and in Table \ref{tab:equant}, which provides complementary quantitative comparison) show that, with relaxations on the scattering parameters, LEGS-FCN achieves better classification accuracy than both LEGS-FIXED and GCN (see Table~\ref{tab:full_table}) while also retaining a more descriptive embedding that maintains the global structure of relations between enzyme classes.
We ran LEGS-FIXED and LEGS-FCN on the ENZYMES dataset. LEGS-FCN which allows the diffusion scales to be learned. For comparison, we also ran a standard GCN whose graph embeddings were obtained via mean pooling. To infer enzyme exchange preferences from their embeddings, we followed \citet{gao2019geometric} in defining the distance from an enzyme $e$ to the enzyme class $\text{EC}_j$ as $\text{dist}(e,\text{EC}_j) := \|v_e-\text{proj}_{C_j}(v_e)\|$, where $v_i$ is the embedding of $e$, and $C_j$ is the PCA subspace of the enzyme feature vectors within $\text{EC}_j$. The distance between the enzyme classes $\text{EC}_i$ and $\text{EC}_j$ is the average of the individual distances, $\text{mean}\{\text{dist}(e, \text{EC}_j):e\in \text{EC}_i\}$. From here, the affinity between two enzyme classes is computed as $\text{pref}(\text{EC}_i,\text{EC}_j)=w_i / \min(\frac{D_{i,i}}{D_{i,j}},\frac{D_{j,j}}{D_{j,i}})$, where $w_i$ is the percentage of enzymes in class $i$ which are closer to another class than their own, and $D_{i,j}$ is the distance between $\text{EC}_i$ and $\text{EC}_j$. 

\begin{table*}[ht]
    \centering
    \caption{Mean $\pm$ std.\ over test set selection on cross-validated LEGS-RBF Net with reduced training set size.}
\begin{tabular}{lllll}
\toprule
Train, Val, Test \% &  80\%, 10\%, 10\% &  70\%, 10\%, 20\% &  40\%, 10\%, 50\% &  20\%, 10\%, 70\% \\
\midrule
COLLAB           &  75.78 $\pm$ 1.95 &  75.00 $\pm$ 1.83 &  74.00 $\pm$ 0.51 &  72.73 $\pm$ 0.59 \\
DD               &  72.58 $\pm$ 3.35 &  70.88 $\pm$ 2.83 &  69.95 $\pm$ 1.85 &  69.43 $\pm$ 1.24 \\
ENZYMES          &  36.33 $\pm$ 4.50 &  34.17 $\pm$ 3.77 &  29.83 $\pm$ 3.54 &  23.98 $\pm$ 3.32 \\
IMDB-BINARY      &  64.90 $\pm$ 3.48 &  63.00 $\pm$ 2.03 &  63.30 $\pm$ 1.27 &  57.67 $\pm$ 6.04 \\
IMDB-MULTI       &  41.93 $\pm$ 3.01 &  40.80 $\pm$ 1.79 &  41.80 $\pm$ 1.23 &  36.83 $\pm$ 3.31 \\
MUTAG            &  33.51 $\pm$ 4.34 &  33.51 $\pm$ 1.14 &  33.52 $\pm$ 1.26 &  33.51 $\pm$ 0.77 \\
NCI1             &  74.26 $\pm$ 1.53 &  74.38 $\pm$ 1.38 &  72.07 $\pm$ 0.28 &  70.30 $\pm$ 0.72 \\
NCI109           &  72.47 $\pm$ 2.11 &  72.21 $\pm$ 0.92 &  70.44 $\pm$ 0.78 &  68.46 $\pm$ 0.96 \\
PROTIENS         &  70.89 $\pm$ 3.91 &  69.27 $\pm$ 1.95 &  69.72 $\pm$ 0.27 &  68.96 $\pm$ 1.63 \\
PTC              &  57.26 $\pm$ 5.54 &  57.83 $\pm$ 4.39 &  54.62 $\pm$ 3.21 &  55.45 $\pm$ 2.35 \\
REDDIT-BINARY    &  86.10 $\pm$ 2.92 &  86.05 $\pm$ 2.51 &  85.15 $\pm$ 1.77 &  83.71 $\pm$ 0.97 \\
REDDIT-MULTI-12K &  38.47 $\pm$ 1.07 &  38.60 $\pm$ 0.52 &  37.55 $\pm$ 0.05 &  36.65 $\pm$ 0.50 \\
REDDIT-MULTI-5K  &  47.83 $\pm$ 2.61 &  47.81 $\pm$ 1.32 &  46.73 $\pm$ 1.46 &  44.59 $\pm$ 1.02 \\
\bottomrule
\end{tabular}
    
    \label{tab:training_reduction}
\end{table*}

\vspace{2pt}\noindent\textbf{Robustness to reduced training set size.} We remark that similar to the robustness shown in~\citet{gao2019geometric} for handcrafted scattering, LEGS-based networks are able to maintain accuracy even when the training set size is shrunk to as low as 20\% of the dataset, with a median decrease of 4.7\% accuracy as when 80\% of the data is used for training, as discussed in the supplement (see Table~\ref{tab:training_reduction}).
\begin{table}[ht]
\caption{CASP GDT regression error over three seeds}
\centering
\begin{small}
\begin{tabular}{lll}
\toprule
($\mu \pm \sigma$) &                   Train MSE &                     Test MSE \\
\midrule
LEGS-FCN &  \textbf{134.34 $\pm$ 8.62} &  \textbf{144.14 $\pm$ 15.48} \\
LEGS-RBF  &           140.46 $\pm$ 9.76 &           152.59 $\pm$ 14.56 \\
LEGS-FIXED  &          136.84 $\pm$ 15.57 &            160.03 $\pm$ 1.81 \\
GCN       &          289.33 $\pm$ 15.75 &           303.52 $\pm$ 18.90 \\
GraphSAGE &          221.14 $\pm$ 42.56 &           219.44 $\pm$ 34.84 \\
GIN       &          221.14 $\pm$ 42.56 &           219.44 $\pm$ 34.84 \\
Baseline  &           393.78 $\pm$ 4.02 &           402.21 $\pm$ 21.45 \\
\bottomrule
\end{tabular}
\end{small}
\label{tab:casp}
\vspace{-5mm}
\end{table}

\begin{figure}[ht]
    \centering
    \includegraphics[width=0.8\linewidth]{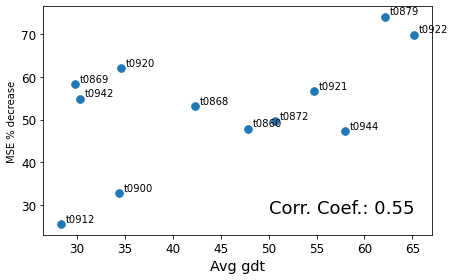}
    \vspace{-3mm}
    \caption{CASP dataset LEGS-FCN \% improvement over GCN in MSE of GDT prediction vs. Average GDT score.}
    \label{fig:casp}
\end{figure}

\begin{table}
    \caption{Mean $\pm$ std.\ over four runs of mean squared error over 19 targets for the QM9 dataset, lower is better.}
    \centering
    \begin{tabular}{ll}
\toprule
$(\mu \pm \sigma)$ & Test MSE                    \\
\midrule
LEGS-FCN    &  \textbf{0.216 $\pm$ 0.009} \\
LEGS-FIXED  &  0.228 $\pm$ 0.019 \\
GraphSAGE   &  0.524 $\pm$ 0.224 \\
GCN         &  0.417 $\pm$ 0.061 \\
GIN         &  0.247 $\pm$ 0.037 \\
Baseline    &  0.533 $\pm$ 0.041 \\
\bottomrule
\end{tabular}

    \label{tab:qm9}
    \vspace{-5mm}
\end{table}

\subsection{Graph Regression}
We next evaluate learnable scattering on two graph regression tasks, the QM9~\citep{gilmer_neural_2017} graph regression dataset, and a new task from the critical assessment of structure prediction (CASP) challenge~\citep{moult_critical_2018}. In the CASP task, the main objective is to score protein structure prediction/simulation models in terms of the discrepancy between their predicted structure and the actual structure of the protein (which is known a priori). The accuracy of such 3D structure predictions are evaluated using a variety of metrics, but we focus on the global distance test (GDT) score~\citep{modi_assessment_2016}. The GDT score measures the similarity between tertiary structures of two proteins with amino-acid correspondence. A higher score means two structures are more similar. For a set of predicted 3D structures for a protein, we would like to quantify their quality by the GDT score.

For this task we use the CASP12 dataset~\citep{moult_critical_2018} and preprocess it similarly to \citet{ingraham_generative_2019}, creating a KNN graph between proteins based on 3D coordinates of each amino acid. From this KNN graph we regress against the GDT score. We evaluate on 12 proteins from the CASP12 dataset and choose random (but consistent) splits with 80\% train, 10\% validation, and 10\% test data out of 4000 total structures. We are interested in structure similarity and use no nonstructural node features.

\vspace{2pt}\noindent\textbf{LEGS outperforms on all CASP targets.} Across all CASP targets we find that LEGS-based architectures significantly outperforms GNN and baseline models. This performance improvement is particularly stark on the easiest structures (measured by average GDT) but is consistent across all structures. In Fig.~\ref{fig:casp} we show the relationship between percent improvement of LEGS over the GCN model and the average GDT score across the target structures. We draw attention to target t0879, where LEGS shows the greatest improvement over other methods. Interestingly this target has particularly long-range dependencies~\citep{ovchinnikov_protein_2018}. Since other methods are unable to model these long-range connections, this suggests LEGS is particularly important on these more difficult to model targets.

\vspace{2pt}\noindent\textbf{LEGS outperforms on the QM9 dataset.} We evaluate the performance of LEGS-based networks on the quantum chemistry dataset QM9~\citep{gilmer_neural_2017}, which consists of 130,000 molecules with $\sim$18 nodes per molecule. We use the node features from \citet{gilmer_neural_2017}, with the addition of eccentricity and clustering coefficient features, and ignore the edge features. We whiten all targets to have zero mean and unit standard deviation. We train each network against all 19 targets as provided in the PyTorch geometric package~\cite{fey_fast_2019}, which includes the targets from \citep{gilmer_neural_2017} and \citet{wu_moleculenet_2018}.
and evaluate the mean squared error on the test set with mean and std.\ over four runs finding that LEGS improves over existing models (Table~\ref{tab:qm9}). On more difficult targets (i.e., those with large test error) LEGS-FCN is able to perform better, where on easy targets GIN is the best. Overall, scattering features offer a robust signal over many targets, and while perhaps less flexible (by construction), they achieve good average performance with significantly fewer parameters.

\begin{table}
    \centering
    \caption{Mean $\pm$ std.\ over four runs of mean squared error over 10 properties in ZINC, lower is better.}
    \scalebox{0.93}{
        
\begin{tabular}{lccc}
\toprule
$(\mu \pm \sigma)$                    & BBAB & FBAB & JBCD \\
\midrule
LEGS-FCN        & $0.591 \pm 0.026$ & \textbf{0.472 $\pm$ 0.018} & $0.603 \pm 0.022$ \\
LEGS-ATTN-FCN   & \textbf{0.551 $\pm$ 0.033} & $0.510 \pm 0.029 $ & \textbf{0.598 $\pm$ 0.072}         \\
LEGS-FIXED      & $0.548 \pm 0.017$ & $0.496 \pm 0.015$ & $0.685 \pm 0.017$ \\
MP-LEGS-FCN     & $1.033 \pm 0.081$ & $0.795 \pm 0.024$ & $0.802 \pm 0.093$ \\
GraphSAGE       & $1.004 \pm 0.037$ & $0.994 \pm 0.026$ & $0.987 \pm 0.011$ \\
GCN             & $0.841 \pm 0.025$ & $0.923 \pm 0.037$ & $0.892 \pm 0.039$ \\
GIN             & $0.670 \pm 0.018$ & $0.673 \pm 0.024$ & $0.689 \pm 0.022$ \\
Baseline        & $1.232 \pm 0.146$ & $1.399 \pm 0.163$ & $1.357 \pm 0.154$ \\
\bottomrule
\end{tabular}

    }
    \label{tab:zinc}
    \vspace{-5mm}
\end{table}
\vspace{2pt}\noindent\textbf{LEGS outperforms on the ZINC15 dataset.} We compared the performance of LEGS-based networks against various architectures using $3$ tranches of the ZINC15 \cite{irwin_zinc_2005} dataset in a multi-property prediction task (Tab.~\ref{tab:zinc}). The BBAB, FBAB and JBCD tranches contain molecules with molecular weight in the range of 200-250, 350-375 and 450-500 Daltons respectively. All networks were trained using the PyTorch geometric package~\cite{fey_fast_2019}, using a 1-hot encoding of atom type as the node signal. We predicted 10 chemical, physical and structural properties for each molecule, including measures of size, shape, lipophilicity, hydrogen bonding, and polarity. LEGS-FCN and LEGS-ATTN-FCN (with an attention layer between geometric scattering and the fully connected regression network) achieve better performance compared to non-learnable scattering (LEGS-FIXED), learnable scattering applied to hidden node representations from a 3-layer graph messaging passing network (MP-LEGS-FCN), and other graph neural networks.

\begin{table}
    \caption{Mean $\pm$ std.\ over four runs of mean squared error in binding affinity prediction of targets in BindingDB, lower is better.}
    \centering
    \begin{tabular}{lcc}
\toprule
$(\mu \pm \sigma)$                & P00918 & P14416 \\
\midrule
LEGS-FCN        & \textbf{0.0318 $\pm$ 0.0009} & $0.0441 \pm 0.0013$  \\
LEGS-ATTN-FCN   & $0.0332 \pm 0.0016$ & \textbf{0.0424 $\pm$ 0.0012} \\
LEGS-FIXED      & $0.0597 \pm 0.0017$ & $0.0514 \pm 0.0020$  \\
MP-LEGS-FCN     & $0.1991 \pm 0.0014$ & $0.1429 \pm 0.0034$  \\
GraphSAGE       & $0.1083 \pm 0.0074$ & $0.1106 \pm 0.0087$  \\
GCN             & $0.1072 \pm 0.0053$ & $0.0994 \pm 0.0039$  \\
GIN             & $0.0615 \pm 0.0022$ & $0.0583 \pm 0.0036$  \\
Baseline        & $0.1137 \pm 0.0076$ & $0.1201 \pm 0.0068$  \\
\bottomrule
\end{tabular}
    \label{tab:bindingdb}
    \vspace{-5mm}
\end{table}
\vspace{2pt}\noindent\textbf{LEGS outperforms on the BindingDB dataset.} We predicted the inhibition coefficient for ligands binding to 2 different targets in the BindingDB dataset \cite{gilson_bindingdb_2016}, namely D(2) dopamine receptor (UniProtKB ID: P14416) and Carbonic anhydrase 2 (UniProtKB ID: P00918). The molecular weight and structural diversity of these ligands was significantly higher compared to the ZINC15 tranches. Learnable scattering networks (LEGS-FCN and LEGS-ATTN-FCN) outperformed GNN and baseline models (Tab.~\ref{tab:bindingdb}). Applying attention to the scattering output did not result in a significant performance improvement over the LEGS-FCN model, perhaps due to the FCN being capable of identifying scattering features pertaining to parts of ligand responsible for binding. 

\begin{table}[ht]
    \centering
    \caption{Mean $\pm$ std.\ test set accuracy on biochemical and social network datasets.}
    \scalebox{0.85}{
        \begin{tabular}{llll}
\toprule
$(\mu \pm \sigma)$ &                GCN &                   GCN-LEGS-FIXED &                   GCN-LEGS-FCN \\
\midrule
DD               &   67.82 $\pm$ 3.81 &  \textbf{74.02 $\pm$ 2.79} &           73.34 $\pm$ 3.57 \\
ENZYMES          &   31.33 $\pm$ 6.89 &           31.83 $\pm$ 6.78 &  \textbf{35.83 $\pm$ 5.57} \\
MUTAG            &   79.30 $\pm$ 9.66 &           82.46 $\pm$ 7.88 &  \textbf{83.54 $\pm$ 9.39} \\
NCI1             &   60.80 $\pm$ 4.26 &           70.80 $\pm$ 2.27 &  \textbf{72.21 $\pm$ 2.32} \\
NCI109           &   61.30 $\pm$ 2.99 &           68.82 $\pm$ 1.80 &  \textbf{69.52 $\pm$ 1.99} \\
PROTEINS         &   74.03 $\pm$ 3.20 &           73.94 $\pm$ 3.88 &  \textbf{74.30 $\pm$ 3.41} \\
PTC              &  56.34 $\pm$ 10.29 &  \textbf{58.11 $\pm$ 6.06} &           56.64 $\pm$ 7.34 \\
COLLAB           &   73.80 $\pm$ 1.73 &  \textbf{76.60 $\pm$ 1.75} &           75.76 $\pm$ 1.83 \\
IMDB-BINARY      &   47.40 $\pm$ 6.24 &           65.10 $\pm$ 3.75 &  \textbf{65.90 $\pm$ 4.33} \\
IMDB-MULTI       &   39.33 $\pm$ 3.13 &  \textbf{39.93 $\pm$ 2.69} &           39.87 $\pm$ 2.24 \\
REDDIT-BINARY    &   81.60 $\pm$ 2.32 &           86.90 $\pm$ 1.90 &  \textbf{87.00 $\pm$ 2.36} \\
REDDIT-MULTI-12K &   42.57 $\pm$ 0.90 &           45.41 $\pm$ 1.24 &  \textbf{45.55 $\pm$ 1.00} \\
REDDIT-MULTI-5K  &   52.79 $\pm$ 2.11 &  \textbf{53.87 $\pm$ 2.75} &           53.41 $\pm$ 3.07 \\
\bottomrule
\end{tabular}
    }
    \label{tab:ensemble}
\end{table}

\vspace{2pt}\noindent\textbf{Ensembling LEGS with GCN features improves classification.}
Recent work by \citet{min2020scattering} combines the features from a fixed scattering transform with a GCN network, showing that this has empirical advantages in semi-supervised node classification, and theoretical representation advantages over a standard \cite{kipf_semi-supervised_2016} style GCN. We ensemble the learned features from a learnable scattering network (LEGS-FCN) with those of GCN and compare this to ensembling fixed scattering features with GCN as in \citet{min2020scattering}, as well as the solo features. Our setting is slightly different in that we use the GCN features from pretrained networks, only training a small 2-layer ensembling network on the combined graph level features. This network consists of a batch norm layer, a 128 width fully connected layer, a leakyReLU activation, and a final classification layer down to the number of classes. In Table~\ref{tab:ensemble} we see that combining GCN features with fixed scattering features in LEGS-FIXED or learned scattering features in LEGS-FCN always helps classification. Learnable scattering features help more than fixed scattering features overall and particularly in the biochemical domain.

\section{Conclusion}\label{sec:conclusion}

Here we introduced a flexible geometric scattering module, that serves as an alternative to standard graph neural network architectures and is capable of learning rich multi-scale features. Our learnable geometric scattering module allows a task-dependent network to choose the appropriate scales of the multiscale graph diffusion wavelets that are part of the geometric scattering transform. We show that incorporation of this module yields improved performance on graph classification and regression tasks, particularly on biochemical datasets, while keeping strong guarantees on extracted features. This also opens the possibility to provide additional flexibility to the module to enable node-specific or graph-specific tuning via attention mechanisms, which are an exciting future direction, but out of scope for the current work. %

\section{Acknowledgments}
This research was partially funded by IVADO Professor funds, CIFAR AI Chair, and NSERC Discovery grant 03267 [\emph{G.W.}]; Chan-Zuckerberg Initiative grants 182702 \& CZF2019-002440 [\emph{S.K.}]; NSF career grant 2047856 [\emph{S.K.}]; Sloan Fellowship FG-2021-15883 [\emph{S.K.}]; NIH grants R01GM135929 \& R01GM130847 [\emph{G.W., S.K.}]; and a Yale-Boehringer Ingelheim Biomedical Data Science Fellowship [\emph{D.B.}]. The content provided here is solely the responsibility of the authors and does not necessarily represent the official views of the funding agencies.

\bibliographystyle{IEEEbib}
\bibliography{main}

\clearpage
\section{Biography Section}

\begin{IEEEbiography}[{\includegraphics[trim=2.9in 1.2in 0.9in 0.1in,width=1in,clip,keepaspectratio]{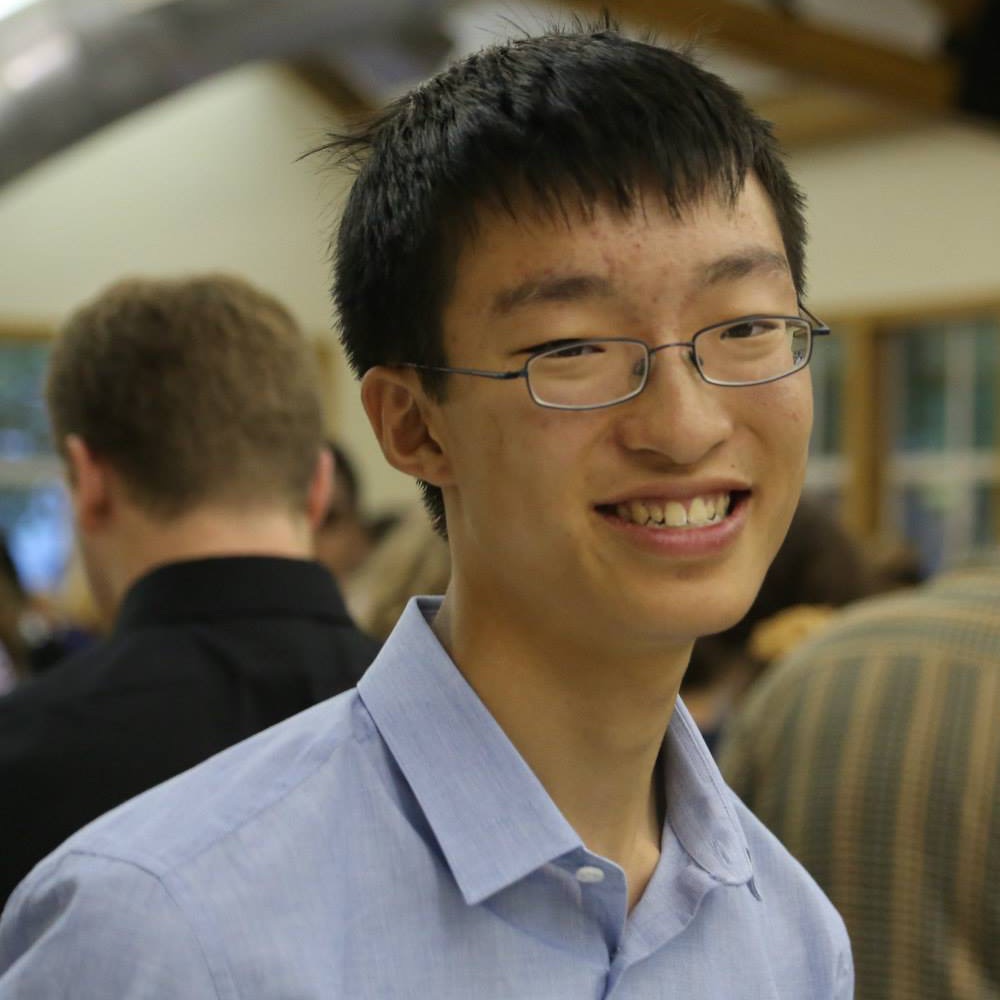}}]{Alexander Tong} is a postdoctoral fellow in the Department of Computer Science and Operations Research (DIRO) at the Université de Montréal and is also affiliated with Mila (the Quebec AI institute). He holds an M.Phil.\ and Ph.D.\ in computer science from Yale University. His research interests are in machine learning and algorithms. His research focuses on optimal transport and neural network methods for geometric single-cell data using a combination of graph and deep learning methods.
\end{IEEEbiography}

\begin{IEEEbiography}[{\includegraphics[trim=2.9in 1.2in 0.9in 0.1in,width=1in,clip,keepaspectratio]{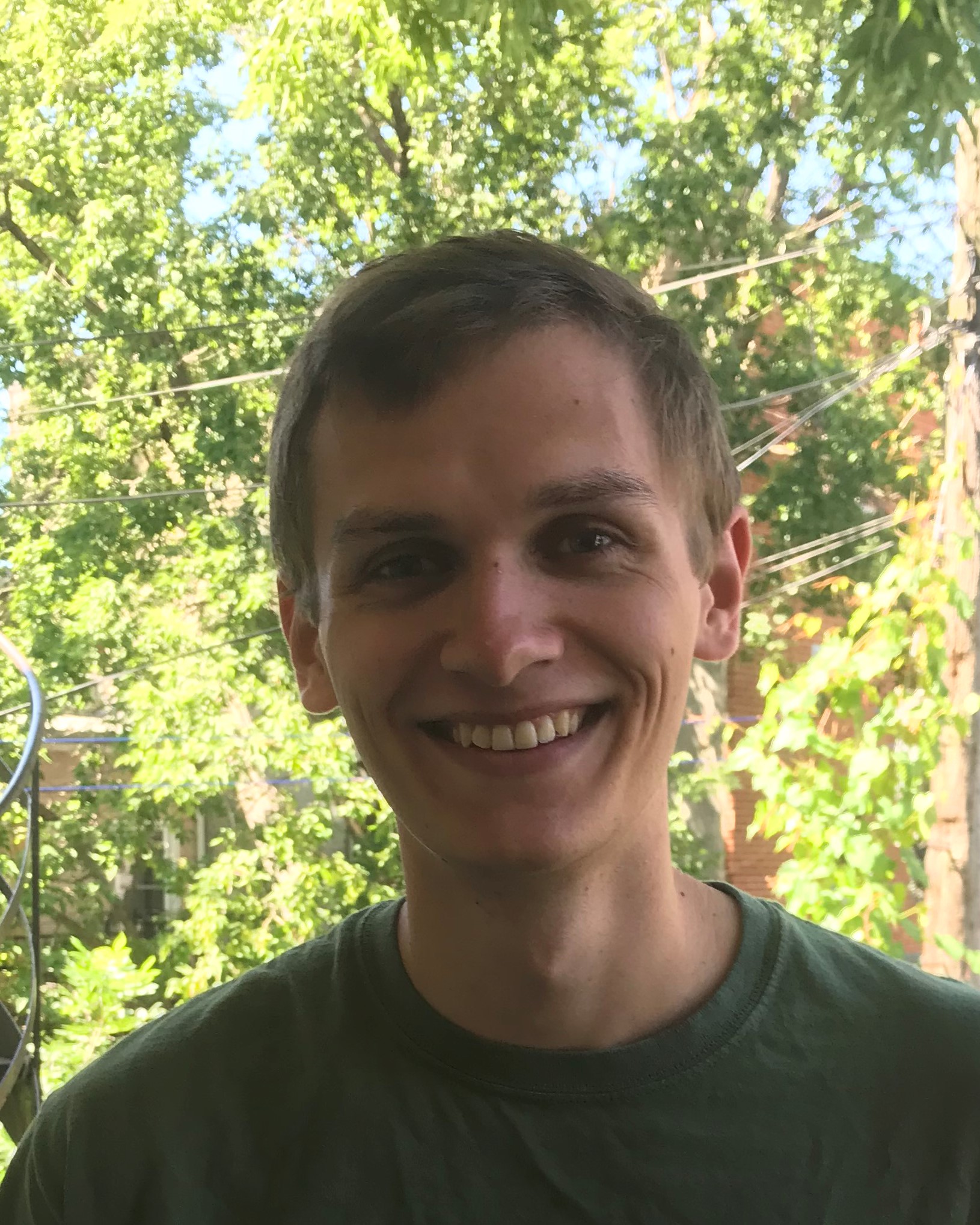}}]{Frederik Wenkel} received the B.Sc.\ degree in Mathematics and the M.Sc.\ degree in Mathematics at Technical University of Munich, in 2019. He is currently a Ph.D.\ candidate in Applied Mathematics at Universit\'{e} de Montr\'{e}al (UdeM) and Mila (the Quebec AI institute), working on geometric deep learning. In particular, he is interested in graph neural networks and their applications in domains such as social networks, bio-chemistry and finance.
\end{IEEEbiography}

\begin{IEEEbiography}[{\includegraphics[width=1in,clip,keepaspectratio]{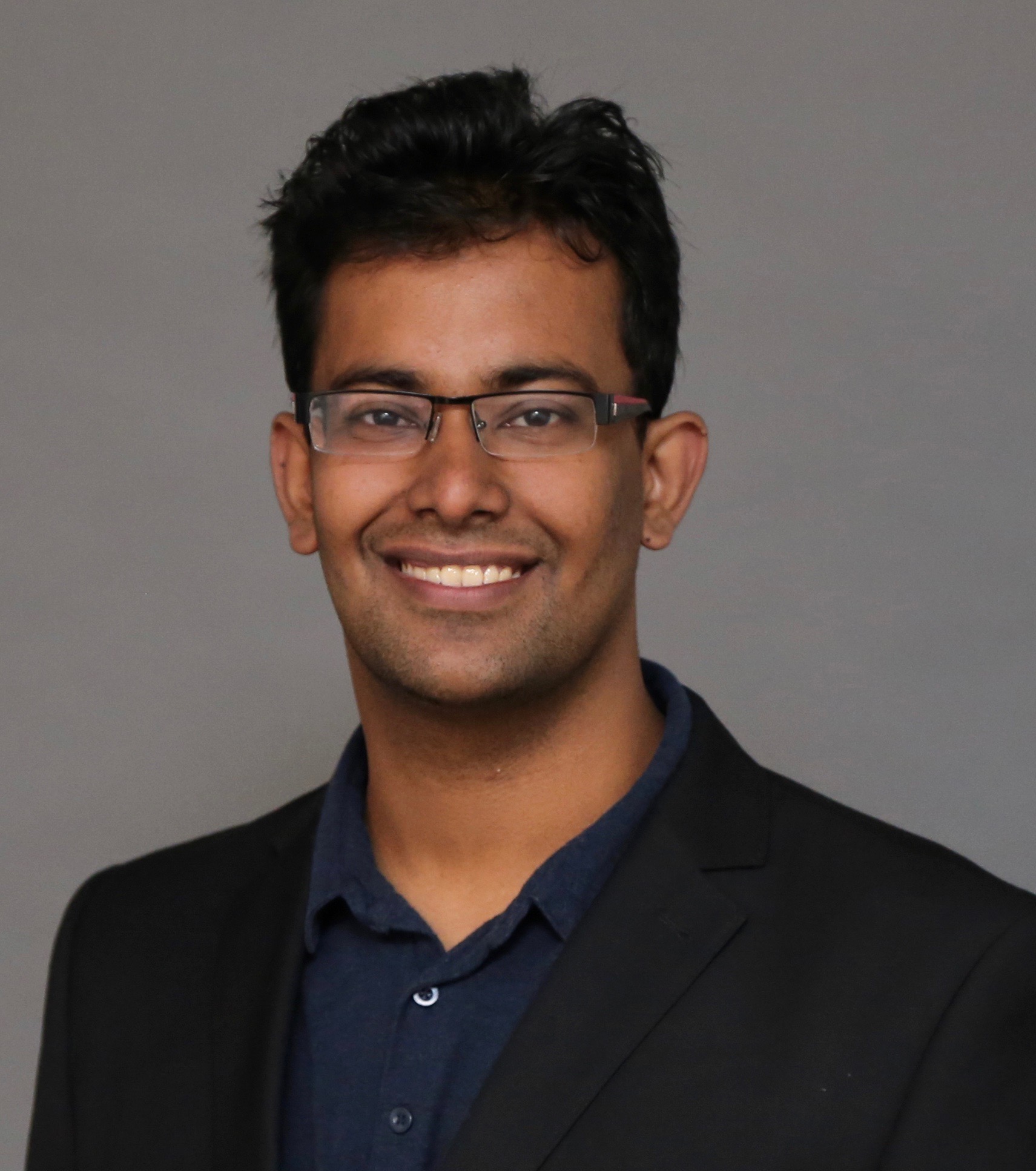}}]{Dhananjay Bhaskar} is a Postdoctoral Research Associate in the Department of Genetics at Yale University and a Yale-Boehringer Ingelheim Biomedical Data Science Fellow. He earned his Ph.D. in Biomedical Engineering from Brown University in 2021. His research combines mathematical modeling, topological data analysis and machine learning for the analysis of multimodal and multiscale data in a variety of biological systems.
\end{IEEEbiography}

\begin{IEEEbiography}[{\includegraphics[width=1in,clip,keepaspectratio]{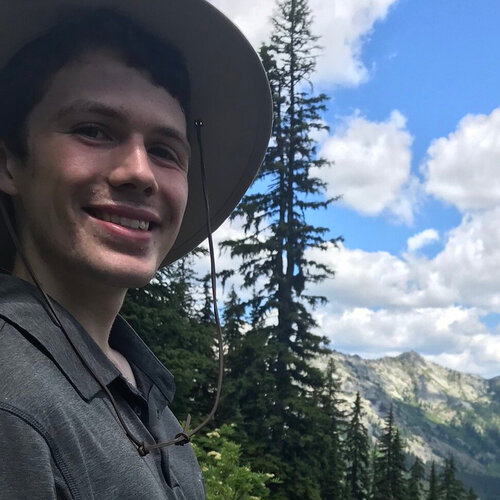}}]{Kincaid MacDonald}
is an undergraduate at Yale studying mathematics, computer science, philosophy, and exploring their practical intersection in artificial intelligence. He is particularly interested in developing mathematical tools to shine light into the ``black box'' of modern deep neural networks.
\end{IEEEbiography}

\begin{IEEEbiography}[{\includegraphics[width=1in,clip,keepaspectratio]{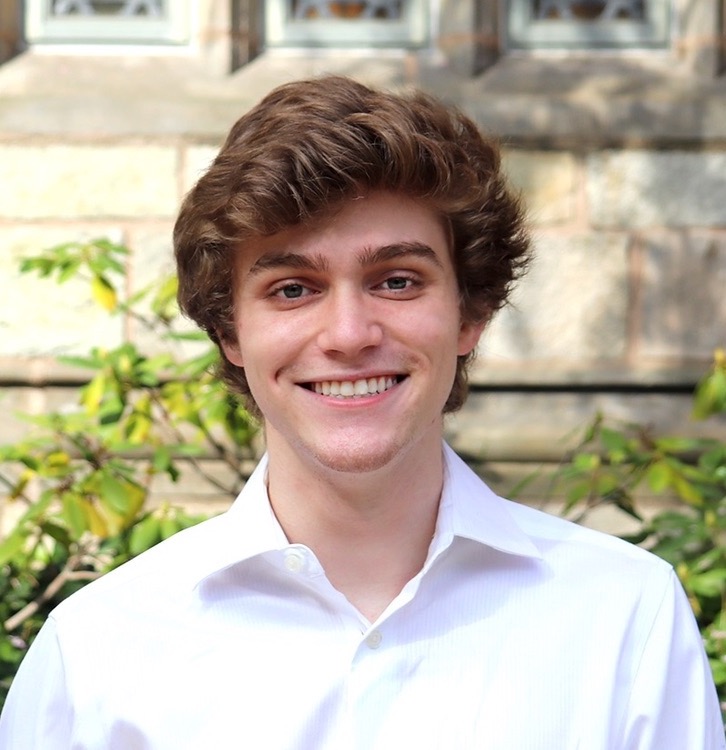}}]{Jackson Grady} is an undergraduate at Yale College currently pursuing a B.S. in Computer Science. He joined the Krishnaswamy Lab in Summer 2021 to explore his research interests in machine learning. His research focuses on informative representations of graph-structured data and its applications to tackling problems such as drug discovery and analysis of protein folding.
\end{IEEEbiography}

\begin{IEEEbiography}[{\includegraphics[width=1in,clip,keepaspectratio]{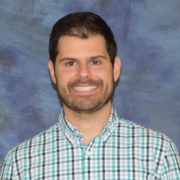}}]{Michael Perlmutter} is a Hedrick Assistant Adjunct Professor in the Department of Mathematics at the University of California Los Angeles. Previously, he has held postdoctoral positions in the Department of Computational Mathematics, Science and Engineering at Michigan State University and in the Department of Statistics and Operations Research at the University of North Carolina at Chapel Hill. He earned his Ph.D. in Mathematics from Purdue University in 2016. His research uses the methods of applied probability and computational harmonic analysis to develop an analyze new methods for data sets with geometric structure.
\end{IEEEbiography}

\begin{IEEEbiography}[{\includegraphics[trim=0.8in 0in 0.8in 0in,width=1in,clip,keepaspectratio]{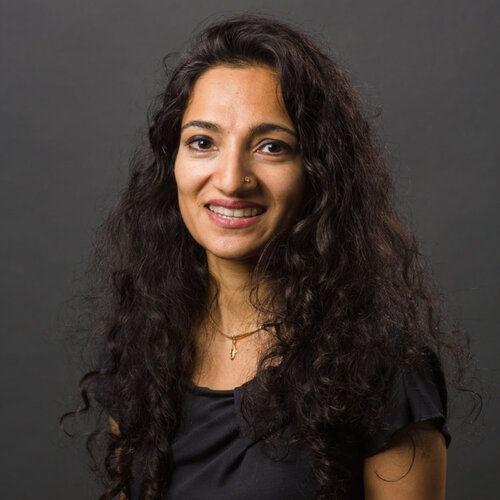}}]{Smita Krishnaswamy} is an Associate Professor in the Department of Genetics at the Yale School of Medicine and in the Department of Computer Science in the Yale School of Applied Science and Engineering, and a core member of the Program in Applied Mathematics. She is also affiliated with the Yale Center for Biomedical Data Science, Yale Cancer Center, and Program in Interdisciplinary Neuroscience. Her research focuses on developing unsupervised machine learning methods (especially graph signal processing and deep-learning) to denoise, impute, visualize and extract structure, patterns and relationships from big, high throughput, high dimensional biomedical data. Her methods have been applied variety of datasets from many systems including embryoid body differentiation, zebrafish development, the epithelial-to-mesenchymal transition in breast cancer, lung cancer immunotherapy, infectious disease data, gut microbiome data and patient data.
She was trained as a computer scientist with a Ph.D.\ from the University of Michigan’s EECS department where her research focused on algorithms for automated synthesis and probabilistic verification of nanoscale logic circuits. Following her time in Michigan, she spent 2 years at IBM’s TJ Watson Research  Center as a researcher in the systems division where she worked on automated bug finding and error correction in logic.
\end{IEEEbiography}

\begin{IEEEbiography}[{\includegraphics[trim=0in 1.5in 0in 0.5in,width=1in,clip,keepaspectratio]{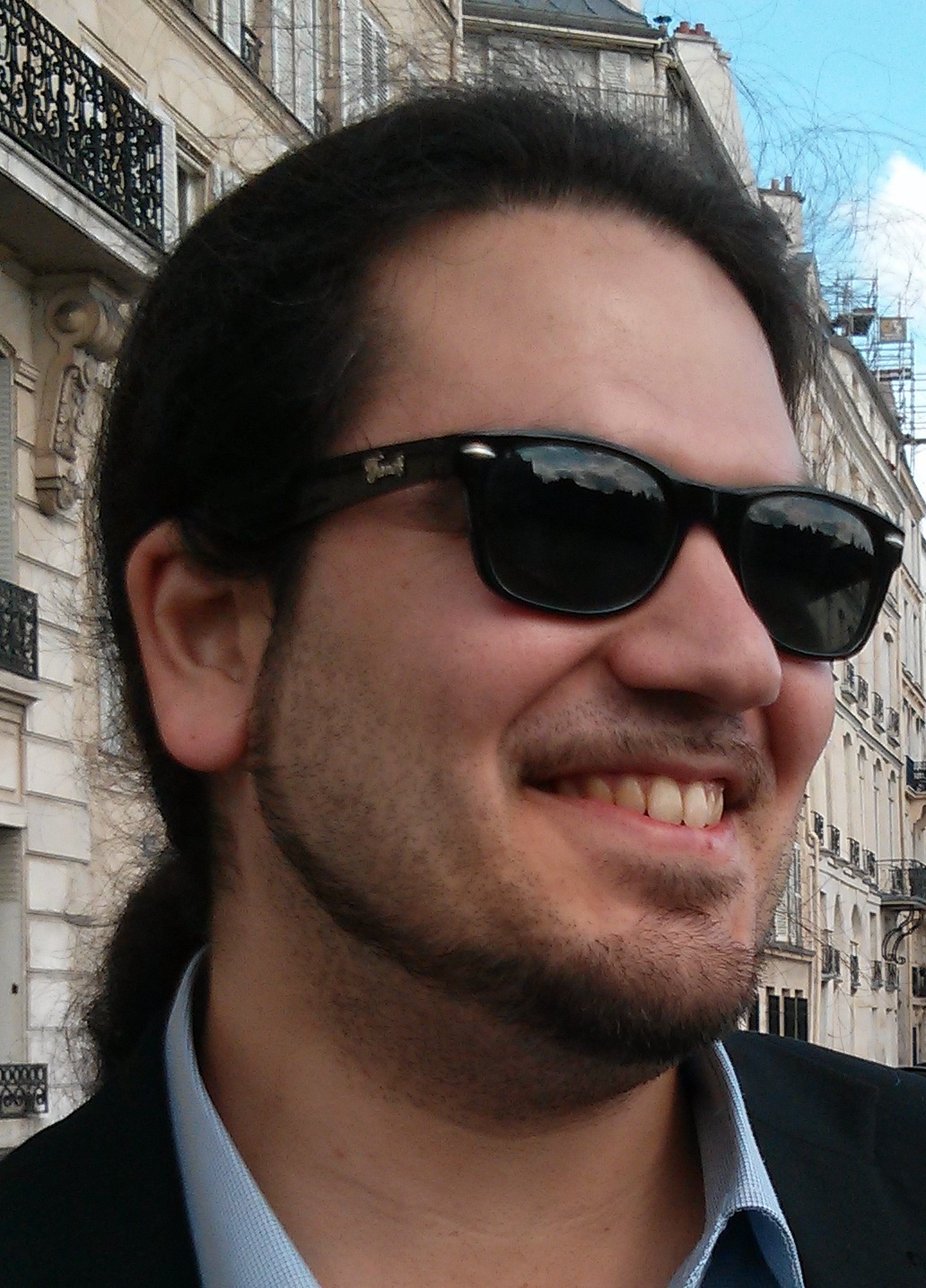}}]{Guy Wolf}
is an associate professor in the Department of Mathematics and Statistics (DMS) at the Universit\'{e} de Montr\'{e}al (UdeM), a core academic member of Mila (the Quebec AI institute), and holds a Canada CIFAR AI Chair. He is also affiliated with the CRM center of mathematical sciences and the IVADO institute of data valorization. He holds an M.Sc.\ and a Ph.D.\ in computer science from Tel Aviv University. Prior to joining UdeM in 2018, he was a postdoctoral researcher (2013-2015) in the Department of Computer Science at \'{E}cole Normale Sup\'{e}rieure in Paris (France), and a Gibbs Assistant Professor (2015-2018) in the Applied Mathematics Program at Yale University. His research focuses on manifold learning and geometric deep learning for exploratory data analysis, including methods for dimensionality reduction, visualization, denoising, data augmentation, and coarse graining. Further, he is particularly interested in biomedical data exploration applications of such methods, e.g., in single cell genomics/proteomics and neuroscience.
\end{IEEEbiography}

\section{The Proof of Theorem \ref{thm: random scales}}\label{sec: auxilliary proofs}

\begin{proof}
Since each of the rows of $\mathbf{F}$ sums to one, we may define $\tau_j$, $j=0,\ldots,J$ to be an independent random variables with probability distribution given by $\mathbf{F}_{j,\cdot}$. 
Then, by definition
\begin{align*}
    \mathbf{\widetilde{\Psi}}_j \mathbf{x}= \sum_{t=1}^m \mathbf{F}_{(j,t)} \mathbf{P}^t_\alpha \mathbf{x} - \mathbf{F}_{(j+1,t)} \mathbf{P}^t_\alpha \mathbf{x} = \mathbb{E} \left[\mathbf{P}^{\tau_j}_\alpha \mathbf{x} - \mathbf{P}^{\tau_{j+1}}_\alpha \mathbf{x}\right]
\end{align*} for  for all $1\leq j\leq J-1$. Similarly, we have 
\begin{equation*}
    \mathbf{\widetilde{\Psi}}_0 \mathbf{x} = \mathbb{E}\left[ \mathbf{x} - \mathbf{P}^{\tau_{1}}_\alpha \mathbf{x}\right]\quad\text{and}\quad     \mathbf{\widetilde{\Phi}}_J \mathbf{x} = \mathbb{E} \mathbf{P}^{\tau_{J}}_\alpha \mathbf{x}
\end{equation*}
Therefore, by Jensen's inequality we have 
\begin{align}
    &\sum_{j=0}^J \|\mathbf{\widetilde{\Psi}}_j{\mathbf{x}}\Vert_{\mD^{-\frac{1}{2}}}^2+\|\widetilde{\Phi}_J \mathbf{x}\Vert_{\mD^{-\frac{1}{2}}}^2
    \\=& \|  \mathbb{E}\left[\mathbf{x} - \mathbf{P}^{\tau_{1}}_\alpha \mathbf{x}\right] \Vert_{\mD^{-\frac{1}{2}}}^2\\&\qquad+ \sum_{j=1}^{J-1} \|  \mathbb{E} \left[\mathbf{P}^{\tau_j}_\alpha \mathbf{x} - \mathbf{P}^{\tau_{j+1}}_\alpha \mathbf{x}\right] \Vert_{\mD^{-\frac{1}{2}}}^2 + \|\mathbb{E}\mathbf{P}^{\tau_{J}}_\alpha \mathbf{x}\Vert_{\mD^{-\frac{1}{2}}}^2 \nonumber \\
    \leq& \mathbb{E} \biggl[\|  \mathbf{x} - \mathbf{P}^{\tau_{1}}_\alpha \mathbf{x} \Vert_{\mD^{-\frac{1}{2}}}^2+\sum_{j=0}^J \| \mathbf{P}^{\tau_j}_\alpha \mathbf{x} - \mathbf{P}^{\tau_{j+1}}_\alpha \mathbf{x} \Vert_{\mD^{-\frac{1}{2}}}^2 
    \\&\qquad+ \|\mathbf{P}^{\tau_{J}}_\alpha \mathbf{x}\Vert_{\mD^{-\frac{1}{2}}}^2 \biggr] \label{eqn: bound inside E}.
\end{align}
By  our assumptions on the support of the  rows of $\mathbf{F}$, we have $\tau_j<\tau_{j+1}$ for all $j$ . Therefore, the conditions of Theorem \ref{thm:frame} are satisfied with probability one and so we have 
\begin{align*}
 &\|  \mathbf{x} - \mathbf{P}^{\tau_{1}}_\alpha \mathbf{x} \Vert_{\mD^{-\frac{1}{2}}}^2+\sum_{j=0}^J \| \mathbf{P}^{\tau_j}_\alpha \mathbf{x} - \mathbf{P}^{\tau_{j+1}}_\alpha \mathbf{x} \Vert_{\mD^{-\frac{1}{2}}}^2 +\|\mathbf{P}^{\tau_{J}}_\alpha \mathbf{x}\Vert_{\mD^{-\frac{1}{2}}}^2 
 \\&\qquad\leq \|\mathbf{x}\Vert_{\mD^{-\frac{1}{2}}}^2
\end{align*}
with probability one. Combining this with completes the proof.
\end{proof}

\section{Gradients of the LEGS module}\label{sec:supp:grad}
Here we analyze the gradients of the LEGS module with respect to its outputs. As depicted in Fig.~\ref{fig:arch}, the LEGS module has 3 inputs, $W$, $\Theta$, and $\alpha$, and $J$ permutation equivariant output matrices $\widetilde{W}_\mF = \{\widetilde\mPsi_{j}, \widetilde\mPhi_{J}\}_{j=0}^{J-1}$. Here we compute partial derivatives of $\widetilde{\mPsi}_j$ for $1 \le j \le J - 1$. The other related gradients of $\widetilde\mPsi_0$ and $\widetilde\mPhi_J$ are easily deducible from these. 
The following partial derivatives and an application of the train rule yield the gradients presented in the main paper:
\begin{align}
    \frac{\partial \widetilde{\mPsi}_j}{\partial \mF_{k,t}} &= 
    \begin{cases}
        \mP^t & \text { if } k = j, \\
        -\mP^t & \text{ if } k = j + 1, \\
        \boldsymbol{0}_{n\times n} & \text{ else,}
    \end{cases}\\
    \frac{\partial \widetilde{\mPsi}_j}{\partial \mP^t} &= (\mF_{j, t} - \mF_{j+1, t}), \\
    \frac{\partial \mF_{k,t}}{\mTheta_{k, t}} &= 
    \begin{cases}
        \sigma(\vtheta_j) (1 - \sigma(\vtheta_j)) & \text { if } k = j, \\
        -\sigma(\vtheta_j) \sigma(\vtheta_k) & \text{ else, }
    \end{cases}\\
    \frac{\partial \mP^t}{\mW_{ab}} &= \sum_{k=1}^t \mP^{k-1} (\mJ^{ab} \mD^{-1}) \mP^{t - k}.
\end{align}

\section{Datasets}\label{sup:dataset}
In this section we provide further information and analysis on the individual datasets which relates the composition of the dataset as shown in Table~\ref{tab:dataset_stats} to the relative performance of our models as shown in Table~\ref{tab:full_table}. Datasets used in Tables \ref{tab:zinc} and \ref{tab:bindingdb} are also described here.

\vspace{2pt}\noindent\textbf{DD}~\cite{dobson_distinguishing_2003} is a dataset extracted from the protein data bank (PDB) of 1178 high resolution proteins. The task is to distinguish between enzymes and non-enzymes. Since these are high resolution structures, these graphs are significantly larger than those found in our other biochemical datasets with a mean graph size of 284 nodes with the next largest biochemical dataset with a mean size of 39 nodes.

\vspace{2pt}\noindent\textbf{ENZYMES}~\cite{borgwardt_protein_2005} is a dataset of 600 enzymes divided into 6 balanced classes of 100 enzymes each. As we analyzed in the main text, scattering features are better able to preserve the structure between classes. LEGS-FCN slightly relaxes this structure but improves accuracy from 32 to 39\% over LEGS-FIXED.

\vspace{2pt}\noindent\textbf{NCI1, NCI109}~\cite{wale_comparison_2008} contain slight variants of 4100 chemical compounds encoded as graphs. Each compound is separated into one of two classes based on its activity against non-small cell lung cancer and ovarian cancer cell lines. Graphs in this dataset have 30 nodes with a similar number of edges. This makes for long graphs with high diameter.

\vspace{2pt}\noindent\textbf{PROTEINS}~\cite{borgwardt_protein_2005} contains 1178 protein structures with the goal of classifying enzymes vs. non enzymes. GCN outperforms all other models on this dataset, however the Baseline model, where no structure is used also performs very similarly. This suggests that the graph structure within this dataset does not add much information over the structure encoded in the eccentricity and clustering coefficient.

\vspace{2pt}\noindent\textbf{PTC}~\cite{toivonen_statistical_2003} contains 344 chemical compound graphs divided into two classes based on whether or not they cause cancer in rats. This dataset is very difficult to classify without features however LEGS-RBF and LEGS-FCN are able to capture the long range connections slightly better than other methods.

\vspace{2pt}\noindent\textbf{COLLAB}~\cite{yanardag_deep_2015} contains 5000 ego-networks of different researchers from high energy physics, condensed matter physics or astrophysics. The goal is to determine which field the research belongs to. The GraphSAGE model performs best on this dataset although the LEGS-RBF network performs nearly as well. Ego graphs have a very small average diameter. Thus, shallow networks can perform quite well on them as is the case here.

\vspace{2pt}\noindent\textbf{IMDB}~\cite{yanardag_deep_2015} contains graphs with nodes representing actresses/actors and edges between them if they are in the same move. These graphs are also ego graphs around specific actors. IMDB-BINARY classifies between action and romance genres. IMDB-MULTI classifies between 3 classes. Somewhat surprisingly GS-SVM performs the best with other LEGS networks close behind. This could be due to oversmoothing on the part of GCN and GraphSAGE when the graphs are so small.

\vspace{2pt}\noindent\textbf{REDDIT}~\cite{yanardag_deep_2015} consists of three independent datasets. In REDDIT-BINARY/MULTI-5K/MULTI-12K, each graph represents a discussion thread where nodes correspond to users and there is an edge between two nodes if one replied to the other's comment. The task is to identify which subreddit a given graph came from. On these datasets GCN outperforms other models.

\vspace{2pt}\noindent\textbf{ZINC15}~\cite{irwin_zinc_2005} is a database of small drug-like molecules for virtual screening. The data is organized into 2D tranches consisting of approximately 997 million molecules categorized by molecular weight, solubility (LogP), reactivity and availability for purchase.

\vspace{2pt}\noindent\textbf{BindingDB}~\cite{gilson_bindingdb_2016} is a publicly available database of binding affinities, focusing on interactions between small, drug-like ligands and proteins considered to be candidate drug-targets. BindingDB contains approximately 2.5 million interactions between more than 8,000 proteins and 1 million drug-like molecules.

\section{Training Details}\label{sup:training}

We train all models for a maximum of 1000 epochs with an initial learning rate of $1.0 \times 10^{-4}$ using the ADAM optimizer~\citep{kingma_adam:_2015}. We terminate training if validation loss does not improve for 100 epochs testing every 10 epochs. Our models are implemented with Pytorch~\cite{paszke_pytorch_2019} and Pytorch geometric~\cite{fey_fast_2019}. Models were run on a variety of hardware resources. For all models we use $q=4$ normalized statistical moments for the node to graph level feature extraction and $m = 16$ diffusion scales in line with choices in \cite{gao2019geometric}. Most experiments were run on a $2 \times 18$ core Intel(R) Xeon(R) CPU E5-2697 v4 @ 2.30GHz server with 512GB of RAM equipped with two Nvidia TITAN RTX gpus. 

\subsection{Cross Validation Procedure}\label{sec:sup:cv}
For all datasets we use 10-fold cross validation with 80\% training data 10\% validation data and 10\% test data for each model. We first split the data into 10 (roughly) equal partitions. For each model we take exactly one of the partitions to be the test set and one of the remaining nine to be the validation set. We then train the model on the remaining eight partitions using the cross-entropy loss on the validation for early stopping checking every ten epochs. For each test set, we use majority voting of the nine models trained with that test set. We then take the mean and standard deviation across these test set scores to average out any variability in the particular split chosen. This results in 900 models trained on every dataset. With mean and standard deviation over 10 ensembled models each with a separate test set.

\begin{table}[htb]
\begin{center}
\caption{Quantified distance between the empirically observed enzyme class exchange preferences of \cite{cuesta_classification_2015}}
\label{tab:equant}
\begin{tabular}{rrr}
\toprule
 LEGS-FIXED &   LEGS-FCN &  GCN\\
\midrule
0.132 &  0.146 &  0.155\\
\bottomrule
\end{tabular}
\end{center}
\end{table}
\begin{table*}[thb]
    \centering
    \caption{Mean $\pm$ std.\ over four runs of mean squared error over 19 targets for the QM9 dataset, lower is better.}
    \scalebox{1}{
        \begin{tabular}{lllllll}
\toprule
{} &                    LEGS-FCN &                  LEGS-FIXED &                GCN &          GraphSAGE &                         GIN &           Baseline \\
\midrule
$\mu$  &  \textbf{0.749 $\pm$ 0.025} &           0.761 $\pm$ 0.026 &  0.776 $\pm$ 0.021 &  0.876 $\pm$ 0.083 &           0.786 $\pm$ 0.032 &  0.985 $\pm$ 0.020 \\
$\alpha$  &  \textbf{0.158 $\pm$ 0.014} &           0.164 $\pm$ 0.024 &  0.448 $\pm$ 0.007 &  0.555 $\pm$ 0.295 &           0.191 $\pm$ 0.060 &  0.593 $\pm$ 0.013 \\
$\epsilon_{\textrm{HOMO}}$  &  \textbf{0.830 $\pm$ 0.016} &           0.856 $\pm$ 0.026 &  0.899 $\pm$ 0.051 &  0.961 $\pm$ 0.057 &           0.903 $\pm$ 0.033 &  0.982 $\pm$ 0.027 \\
$\epsilon_{\textrm{LUMO}}$  &           0.511 $\pm$ 0.012 &  \textbf{0.508 $\pm$ 0.005} &  0.549 $\pm$ 0.010 &  0.688 $\pm$ 0.216 &           0.555 $\pm$ 0.006 &  0.805 $\pm$ 0.025 \\
$\Delta \epsilon$  &  \textbf{0.587 $\pm$ 0.007} &  \textbf{0.587 $\pm$ 0.006} &  0.609 $\pm$ 0.009 &  0.755 $\pm$ 0.177 &           0.613 $\pm$ 0.013 &  0.792 $\pm$ 0.010 \\
$\langle R^2 \rangle$  &  \textbf{0.646 $\pm$ 0.013} &           0.674 $\pm$ 0.047 &  0.889 $\pm$ 0.014 &  0.882 $\pm$ 0.118 &           0.699 $\pm$ 0.033 &  0.833 $\pm$ 0.026 \\
$\textrm{ZPVE}$  &           0.018 $\pm$ 0.012 &           0.020 $\pm$ 0.011 &  0.099 $\pm$ 0.011 &  0.321 $\pm$ 0.454 &  \textbf{0.012 $\pm$ 0.006} &  0.468 $\pm$ 0.005 \\
$U_0$  &           0.017 $\pm$ 0.005 &           0.024 $\pm$ 0.008 &  0.368 $\pm$ 0.015 &  0.532 $\pm$ 0.405 &  \textbf{0.015 $\pm$ 0.005} &  0.379 $\pm$ 0.013 \\
$U$  &           0.017 $\pm$ 0.005 &           0.024 $\pm$ 0.008 &  0.368 $\pm$ 0.015 &  0.532 $\pm$ 0.404 &  \textbf{0.015 $\pm$ 0.005} &  0.378 $\pm$ 0.013 \\
$H$  &           0.017 $\pm$ 0.005 &           0.024 $\pm$ 0.008 &  0.368 $\pm$ 0.015 &  0.532 $\pm$ 0.404 &  \textbf{0.015 $\pm$ 0.005} &  0.378 $\pm$ 0.013 \\
$G$ &           0.017 $\pm$ 0.005 &           0.024 $\pm$ 0.008 &  0.368 $\pm$ 0.015 &  0.533 $\pm$ 0.404 &  \textbf{0.015 $\pm$ 0.005} &  0.380 $\pm$ 0.014 \\
$c_{\textrm{v}}$ &  \textbf{0.254 $\pm$ 0.013} &           0.279 $\pm$ 0.023 &  0.548 $\pm$ 0.023 &  0.617 $\pm$ 0.282 &           0.294 $\pm$ 0.003 &  0.631 $\pm$ 0.013 \\
$U_0^{\textrm{ATOM}}$ &           0.034 $\pm$ 0.014 &           0.033 $\pm$ 0.010 &  0.215 $\pm$ 0.009 &  0.356 $\pm$ 0.437 &  \textbf{0.020 $\pm$ 0.002} &  0.478 $\pm$ 0.014 \\
$U^{\textrm{ATOM}}$ &           0.033 $\pm$ 0.014 &           0.033 $\pm$ 0.010 &  0.214 $\pm$ 0.009 &  0.356 $\pm$ 0.438 &  \textbf{0.020 $\pm$ 0.002} &  0.478 $\pm$ 0.014 \\
$H^{\textrm{ATOM}}$ &           0.033 $\pm$ 0.014 &           0.033 $\pm$ 0.010 &  0.213 $\pm$ 0.009 &  0.355 $\pm$ 0.438 &  \textbf{0.020 $\pm$ 0.002} &  0.478 $\pm$ 0.014 \\
$G^{\textrm{ATOM}}$ &           0.036 $\pm$ 0.014 &           0.036 $\pm$ 0.011 &  0.219 $\pm$ 0.009 &  0.359 $\pm$ 0.436 &  \textbf{0.023 $\pm$ 0.002} &  0.479 $\pm$ 0.014 \\
$A$ &           0.002 $\pm$ 0.002 &           0.001 $\pm$ 0.001 &  0.017 $\pm$ 0.034 &  0.012 $\pm$ 0.022 &  \textbf{0.000 $\pm$ 0.000} &  0.033 $\pm$ 0.013 \\
$B$ &           0.083 $\pm$ 0.047 &  \textbf{0.079 $\pm$ 0.033} &  0.280 $\pm$ 0.354 &  0.264 $\pm$ 0.347 &           0.169 $\pm$ 0.206 &  0.205 $\pm$ 0.220 \\
$C$ &  \textbf{0.062 $\pm$ 0.005} &           0.176 $\pm$ 0.231 &  0.482 $\pm$ 0.753 &  0.470 $\pm$ 0.740 &           0.321 $\pm$ 0.507 &  0.368 $\pm$ 0.525 \\
\bottomrule
\end{tabular}
    }
    \label{tab:qm9_split}
\end{table*}

\section{Additional Experiments}

\vspace{2pt}\noindent\textbf{Quantification of the enzyme class exchange preferences}
We quantify the empirically observed enzyme class exchange preferences of \cite{cuesta_classification_2015} and the class exchange preferences inferred from LEGS-FIXED, LEGS-FCN, and a GCN in Table~\ref{tab:equant}. We measure the cosine distance between the graphs represented by the chord diagrams in Figure~\ref{fig:ribbon}. As before, the self-affinities were discarded. LEGS-Fixed reproduces the exchange preferences the best, but LEGS-FCN still reproduces the observed exchange preferences well and has significantly better classification accuracy.

\vspace{2pt}\noindent\textbf{QM9 Target Breakdown}~\cite{gilmer_neural_2017, wu_moleculenet_2018} contains graphs that each represent chemicals with \~18 atoms. Regression targets represent chemical properties of the molecules. These targets are respectively, the dipole moment $\mu$, the isotropic polarizability $\alpha$, the highest occupied molecular orbital energy $\epsilon_{\textrm{HOMO}}$, the lowest unoccupied molecular orbital energy $\epsilon_{\textrm{LUMO}}$, the difference $\Delta \epsilon = \epsilon_{\textrm{HOMO}} - \epsilon_{\textrm{LUMO}}$, the electronic spatial extent, $\langle R^2 \rangle$, the zero point vibrational energy $\textrm{ZPVE}$, the internal energy at 0K $U_0$, the internal energy $U$, the enthalpy $H$, the free energy $G$, and the heat capavity $c_{\textrm{v}}$ at 25C, the atomization energy at 0K $U_0^{\textrm{ATOM}}$ and 25C $U^{\textrm{ATOM}}$, the atomization enthalpy $H^{\textrm{ATOM}}$ and free energy $G^{\textrm{ATOM}}$ at 25C, and three rotational constants $A, B, C$ measured in gigahertz. For more information see references~\cite{gilmer_neural_2017, wu_moleculenet_2018, fey_fast_2019}. In Table~\ref{tab:qm9_split} we split out performance by target. GIN performs slightly better on the molecule energy targets both overall and atomization targets $U_0$, $U$, $H$, and $G$ where both LEGS and GIN significantly outperform the other models GCN, GraphSAGE and the structure invariant baseline. On all other targets, and especially the more difficult targets (measured by baseline MSE) the LEGS module performs the best or near to the best.

\end{document}